\documentclass{article}

\usepackage{microtype}
\usepackage{graphicx}
\usepackage{subcaption}
\usepackage{booktabs} 

\usepackage{hyperref}

\usepackage[preprint]{icml2026}

\usepackage{amsmath}
\usepackage{amssymb}
\usepackage{mathtools}
\usepackage{amsthm}

\usepackage[capitalize,noabbrev]{cleveref}

\usepackage{multirow}
\usepackage{threeparttable}
\usepackage{adjustbox}
\theoremstyle{plain}
\newtheorem{theorem}{Theorem}[section]

\newtheorem{lemma}[theorem]{Lemma}

\theoremstyle{definition}
\newtheorem{definition}[theorem]{Definition}
\newtheorem{assumption}[theorem]{Assumption}
\theoremstyle{remark}

\usepackage[textsize=tiny]{todonotes}

\icmltitlerunning{PowerGenie: Analytically-Guided Evolutionary Discovery of Superior Reconfigurable Power Converters}

\begin{document}

\twocolumn[
  \icmltitle{PowerGenie: Analytically-Guided Evolutionary \\ Discovery of Superior Reconfigurable Power Converters}



  \icmlsetsymbol{equal}{*}

  \begin{icmlauthorlist}
    \icmlauthor{Jian Gao}{NEU}
    \icmlauthor{Yiwei Zou}{RICE}
    \icmlauthor{Abhishek Pradhan}{NEU}
    \icmlauthor{Wenhao Huang}{RICE}
    \icmlauthor{Yumin Su}{RICE}
    \icmlauthor{Kaiyuan Yang}{RICE}
    \icmlauthor{Xuan Zhang}{NEU}
  \end{icmlauthorlist}

    \icmlaffiliation{NEU}{Department of Electrical and Computer Engineering, Northeastern University, Boston, MA, USA}
    \icmlaffiliation{RICE}{Department of Electrical and Computer Engineering, Rice University, Houston, TX, USA}

    \icmlcorrespondingauthor{Xuan Zhang}{xuan.zhang@northeastern.edu}

  \icmlkeywords{Machine Learning, ICML}

  \vskip 0.3in
]



\printAffiliationsAndNotice{}  

\begin{abstract}
Discovering superior circuit topologies requires navigating an exponentially large design space—a challenge traditionally reserved for human experts. Existing AI methods either select from predefined templates or generate novel topologies at a limited scale without rigorous verification, leaving large-scale performance-driven discovery underexplored. We present PowerGenie, a framework for automated discovery of higher-performance reconfigurable power converters at scale. PowerGenie introduces: (1) an automated analytical framework that determines converter functionality and theoretical performance limits without component sizing or SPICE simulation, and (2) an evolutionary finetuning method that co-evolves a generative model with its training distribution through fitness selection and uniqueness verification. Unlike existing methods that suffer from mode collapse and overfitting, our approach achieves higher syntax validity, function validity, novelty rate, and figure-of-merit (FoM). PowerGenie discovers a novel 8-mode reconfigurable converter with 23\% higher FoM than the best training topology. SPICE simulations confirm average absolute efficiency gains of 10\% across 8 modes and up to 17\% at a single mode. Code will be released upon publication.
\end{abstract}

\section{Introduction}
\label{sec: Intro}
\begin{figure}[!t]
\begin{center}
\includegraphics[width=0.85\linewidth]{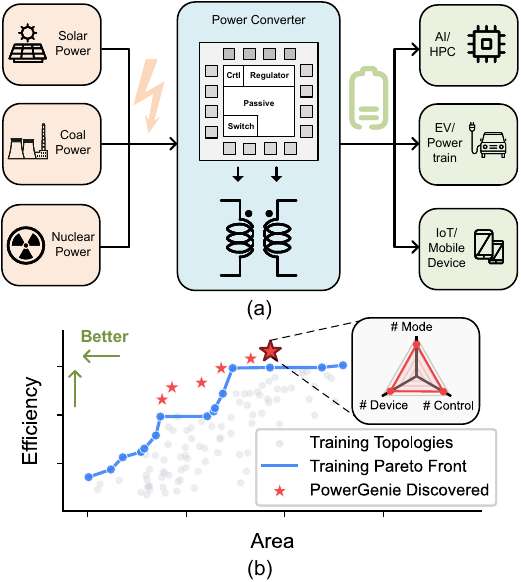}
\caption{(a) Power converters bridge energy sources to applications spanning AI infrastructure, EVs, and mobile devices. (b) PowerGenie discovers topologies beyond the training Pareto front.}
\label{fig: highlevel}
\end{center}
\end{figure}
Artificial intelligence (AI) has fundamentally transformed how scientists and engineers tackle complex problems, achieving superhuman performance in protein structure prediction~\citep{jumper2021highly}, chip placement~\citep{mirhoseini2021graph}, weather forecasting~\citep{lam2023learning}, and mathematical theorem proving~\citep{trinh2024solving}. More recently, AlphaEvolve~\citep{novikov2025alphaevolve} demonstrated that combining large language models with evolutionary search can discover algorithms that surpass decades of human-engineered solutions, including improved matrix multiplication routines and more efficient sorting networks. Inspired by these breakthroughs, analog circuit topology discovery has emerged as a compelling new frontier for AI-driven design automation~\citep{dong2023cktgnn,anonymous2024analoggenie,11240792,hu2025adreamdco,karahan2024deep}. Circuit topology design presents a uniquely challenging problem: it demands not only navigating an exponentially large design space but also uncover novel designs with performance exceeding existing designs. However, current approaches largely focus on either selecting from predefined topology templates or generating novel structures at a limited scale without rigorously verifying their functionality and performance. This leaves a critical question unanswered: \textit{can AI learn from existing circuits and evolve to discover superior designs with higher performance?}

Two key challenges impede large-scale performance-driven circuit topology discovery. The first is the prohibitive cost of sizing and evaluation. Since a circuit's functionality and performance cannot be determined from topology alone, each candidate design requires sizing optimization followed by SPICE simulation—dramatically complicating the overall optimization flow. Although recent work has sought to accelerate sizing~\citep{cao2024rose, budak2021dnn, wang2020gcn, lyu2017efficient, shahane2023graph, krylov2023learning} and evaluation~\citep{poddar2025insight, yang2025zerosim, fan2024graph}, these empirical methods generalize poorly to unseen topologies and struggle to accurately determine theoretical performance limits. The second challenge lies in developing a finetuning strategy that simultaneously enhances generation quality and preserves creative diversity. Alignment is critical for transforming foundation models into specialized agents, yet standard finetuning often leads to mode collapse and loss of diversity~\citep{ramasesh2021effect,kotha2023understanding}. One promising direction involves selecting high-influence samples while iteratively refining dataset diversity. While most prior work finetunes on fixed datasets with various selection or curriculum strategies~\citep{zhou2023lima,xia2024less,li2024quantity, sener2017active, ash2019deep, mirzasoleiman2020coresets, zhou2020curriculum, rafailov2023direct}, recent advances in self-improvement~\citep{wang2023self,shumailov2024ai,xu2024wizardlm,zelikman2022star,schulman2017proximal}—where models expand training sets using their own generations—have shown promise in data-scarce domains. However, these approaches risk model collapse when noise accumulates in synthetic data across iterations.
Analog circuit topology discovery, as an inherently data-scarce problem, demands a more robust closed-loop solution: one that co-evolves both the training set and the model's discovery capability through careful selection during finetuning.

To address these challenges, we propose PowerGenie, a unified framework for large-scale superior power converter topology discovery. PowerGenie introduces a graph-based analytical framework that automatically determines any switched-capacitor (SC) converter's functionality and theoretical performance limits—including slow-switching limit (SSL) and fast-switching limit (FSL) impedances—without requiring component sizing or process design kit (PDK) parameters. While these metrics are standard for comparing converter topologies~\citep{seeman2008analysis,jiang2017digital,karadi20144,jung201612}, prior analyses relied entirely on manual derivation; PowerGenie is the first fully automated framework for arbitrary SC topologies. The framework transforms circuit netlists into directed graphs, constructs fundamental loop matrices via spanning tree decomposition, and extracts voltage conversion ratios and charge multipliers using Tellegen's Theorem~\citep{tellegen1952general}. Building on this foundation, PowerGenie introduces evolutionary finetuning that formulates topology discovery as population-based optimization. Each generation, elite selection preserves top-performing circuits while tournament selection~\citep{miller1995genetic} maintains diversity; the finetuned model serves as a learned variation operator, generating candidates that undergo functional verification and uniqueness checking via graph isomorphism. This establishes a co-evolutionary dynamic: the model learns to generate higher-fitness circuits while the population accumulates superior designs that further refine generative capabilities. Our key contributions are as follows:

\begin{itemize}
\item \textbf{Automated Analytical Framework.} We introduce the first automated analytical framework for SC power converters that determines functionality and theoretical performance limits in seconds rather than hour-long SPICE optimization for rapid performance evaluation.

\item \textbf{Evolutionary Finetuning.} We propose an evolutionary framework that co-evolves the generative model and training distribution through fitness-based selection and uniqueness verification, continuously improving generation quality while preserving diversity.

\item \textbf{Superior Generation Quality.} Evolutionary finetuning improves syntax validity by 41.9\%, functional validity by 45.1\%, novelty rate by 28.6\%, and figure-of-merit (FoM) by 14\% over the best alignment baseline.

\item \textbf{Novel High-Performance Topology.} PowerGenie discovers a novel 8-mode reconfigurable converter with 23\% higher FoM—combining SSL metric, FSL metric, and passive component count—than the best topology in the dataset. SPICE validation confirms average absolute efficiency gains of 10\% across all modes, reaching 17\% gain at individual conversion ratios.
\end{itemize}
\section{Preliminaries and Related Work}
\label{sec: Related}
\subsection{Power Converters: Fundamentals and Applications}
Power converters are essential electronic systems that transform electrical energy between different voltage levels and forms (AC/DC), enabling efficient power delivery from sources to loads. They employ switching devices and passive energy storage elements to transfer power; by periodically toggling switches and redistributing energy, converters regulate output voltage while minimizing losses. This voltage conversion is critical because energy sources and loads rarely share compatible voltage requirements. Power converters are ubiquitous in modern technology—every smartphone, electric vehicle, and AI data center relies on them to efficiently deliver power. Improvements in both efficiency and power density directly translate to longer battery life, extended EV driving range, and enabling more efficient data centers for AI training and inference.

\subsection{Reconfigurable SC Power Converter}
The demand for integrated power conversion with multiple voltage conversion ratios (VCRs) has grown rapidly due to the miniaturization of digital circuits. Modern multi-core processors employ dynamic voltage scaling to balance performance and efficiency~\citep{le2011design}, while IoT devices require on-chip conversion to accommodate wide voltage dynamics~\citep{10904683}. Reconfigurable SC DC-DC converters have emerged as attractive solutions, offering full integration with high efficiency and power density~\citep{sanders2012road}. However, designing multi-VCR converters is considerably more challenging than single-mode designs~\citep{dickson1976chip, lin2003topological}. As shown in Figure~\ref{fig: humandesign}, designers must merge multiple single-mode topologies and identify maximal common sub-topologies to maximize device reuse. Furthermore, since most switches operate across multiple VCR modes, the space of possible control configurations becomes highly complex: a two-phase, single-mode converter requires only 2-bit control configurations, but supporting additional modes causes exponential scaling—a two-phase, two-mode converter requires 4-bit control configurations specifying switch states across all modes and phases. This complexity restricts human experts to few-mode designs~\citep{le2013sub,el201393}, while many-mode solutions rely on digital synthesis with significant area overhead~\citep{10339894}. AI-driven methods offer a promising alternative for exploring the vast design space of many-mode converters.

\subsection{AI for Analog Circuit Topology Design}
Recent advancements in analog topology synthesis focus on assembling predefined circuit structures. Methods such as FALCON~\citep{mehradfar2025falcon}, AnalogCoder~\citep{lai2025analogcoder,lai2025analogcoderpro}, and Artisan~\citep{chen2024artisan,shen2025atelier} leverage expert-verified templates—via selection, code generation, or retrieval-augmented generation—but are fundamentally bounded by them. To transcend these limitations, generative approaches construct novel circuits from scratch: CktGNN~\citep{dong2023cktgnn} uses variational autoencoders to generate circuit graphs from functional subgraphs, LaMAGIC~\citep{chang2024lamagic,chang2025lamagic2} formulates generation as sequence-to-sequence modeling, and AnalogGenie~\citep{anonymous2024analoggenie,gaoanaloggenie} enables large-scale topology generation via Eulerian circuit representations. For performance-driven discovery, reinforcement learning gained traction: AUTOCIRCUIT-RL~\citep{vijayaraghavan2025autocircuit} refines LLM-generated netlists using proximal policy optimization (PPO) guided by proxy reward models, while EVA~\citep{gao2025eva} applies direct preference optimization (DPO) on performance-labeled topologies. However, each faces distinct bottlenecks: AUTOCIRCUIT-RL's reliance on large-scale offline SPICE simulations to train its reward models limits scalability, while EVA's training on fixed datasets leads to mode collapse. Addressing these challenges requires physics-based evaluation to bypass data-intensive training of reward models, and a dynamic training strategy that enriches the dataset while preserving diversity.

\begin{figure}[!t]
\begin{center}
\includegraphics[width=0.8\linewidth]{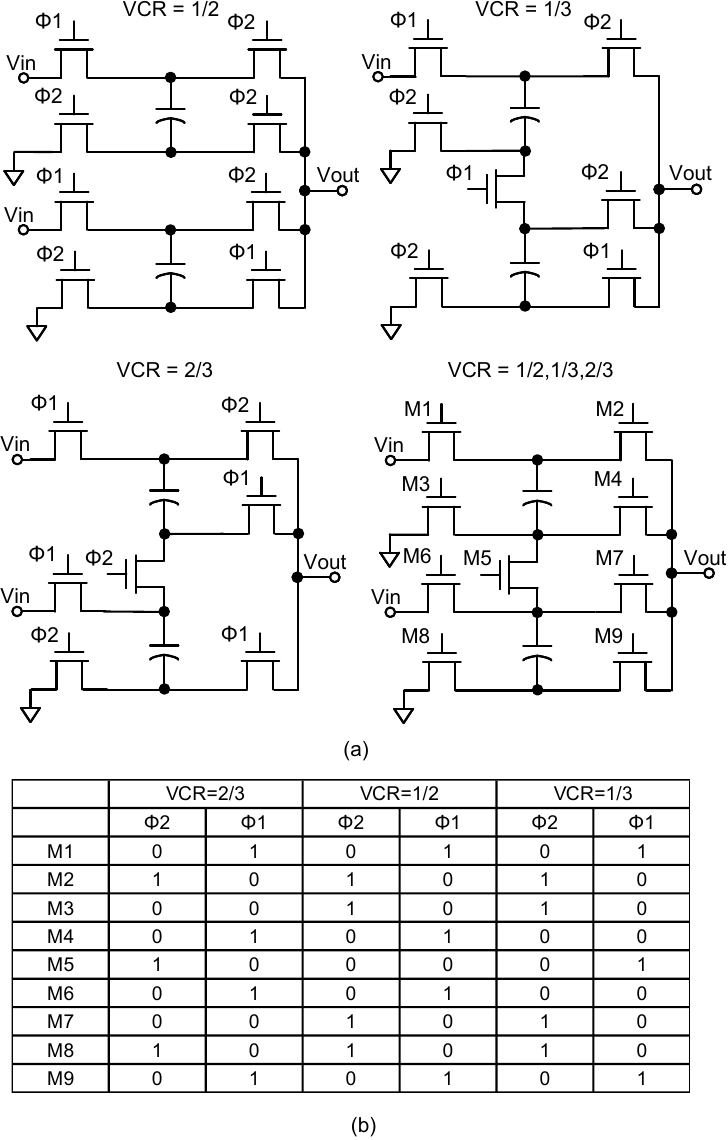}
\caption{(a) A 3-mode reconfigurable SC converter designed by human experts~\citep{le201032nm}, merging three independent single-mode topologies into a unified structure. (b) Corresponding 6-bit control configurations for each MOSFET across VCR modes and phases, where 1 indicates on and 0 indicates off.}
\label{fig: humandesign}
\end{center}
\end{figure}
\begin{figure*}[!t]
\begin{center}
\includegraphics[width=0.9\linewidth]{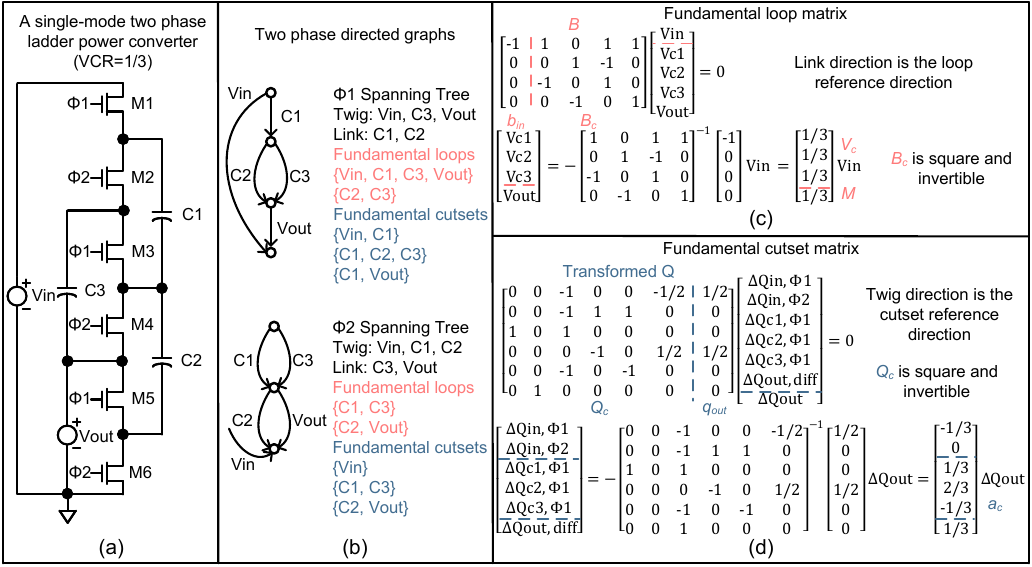}
\caption{Analytical framework demonstration using a single-mode two-phase converter (VCR=1/3). (a) Topology. (b) Phase networks with spanning trees, fundamental loops, and cutsets. (c) VCR and capacitor voltage extraction via KVL. (d) Charge multiplier extraction via KCL. Switch charge multipliers and blocking voltages follow similarly by including switches in the phase graphs (see Appendix~\ref{sec:Q_construction}). For reconfigurable converters, PowerGenie decouples the topology into single-mode sub-topologies and processes each independently.}
\label{fig: analytical}
\end{center}
\end{figure*}

\section{Approach}
\label{sec: Approach}
\subsection{Problem Formulation}
PowerGenie addresses the circuit topology discovery problem: given target voltage conversion ratios, discover novel reconfigurable power converter topologies—along with their control schemes—that achieve superior performance over existing training designs. We focus on reconfigurable converters due to their higher complexity than single-mode designs explored by existing AI methods~\citep{chang2025lamagic2}, and their relevance to power IC research~\citep{fu2025switched}. While typical human designs target 3 to 6 VCRs~\citep{jiang2017digital}, PowerGenie targets 8 VCRs (1/5, 1/4, 1/3, 1/2, 2/3, 3/4, 4/5, 1) to demonstrate scalability. PowerGenie relies on analytical evaluation. It characterizes theoretical performance limits using three metrics: (1) SSL, for capacitor-related losses; (2) FSL, for switch conduction losses; and (3) capacitor count, which dominates area. These combine into a unified FoM. Transistor-level SPICE simulations with gradient-based sizing validate actual performance.

\subsection{Automated Analytical Framework}
An efficient and accurate verifier is essential for performance-driven training and large-scale evaluation. Given a reconfigurable power converter netlist and its control scheme, PowerGenie first extracts the capacitor count, then decouples the topology into single-mode sub-topologies based on control schemes—an eight-mode topology is decomposed into 8 single-mode topologies, each processed independently. As shown in Figure~\ref{fig: analytical}, PowerGenie then applies the following steps to verify functional validity and compute SSL and FSL metrics for each sub-topology. A reconfigurable converter is functional only when all decoupled sub-topologies are valid, with overall SSL and FSL computed as averages across all modes. We first state the operating assumptions under which the analysis is performed.

\begin{assumption}[No-Load Steady State]
  \label{ass:noload}
  The VCR is determined with a DC voltage source $V_{IN}$ at the input, the output open-circuited, and the converter in periodic steady state such that capacitor voltages remain constant.
\end{assumption}

\begin{assumption}[Asymptotic Operating Limits]
  \label{ass:limits}
  SSL analysis assumes ideal switches with impulsive charge transfer; FSL analysis assumes constant capacitor voltages with resistive switch losses dominating.
\end{assumption}

\subsubsection{Determining Functional Validity}
PowerGenie first determines if a converter is properly posed, then extracts its VCR. Under Assumption~\ref{ass:noload}, we model the converter as a directed graph $G = (V, E)$, where nodes $V$ represent electrical connections and edges $E$ represent circuit branches (capacitors, switches, and voltage sources). The switch configurations partition $G$ into two phase networks, $G_{\phi_1}$ and $G_{\phi_2}$. For each phase network, we construct a spanning tree \citep{kruskal1956shortest}---a connected subgraph containing all nodes without forming loops. This partitions branches into \emph{twigs} (tree edges) and \emph{links} (non-tree edges). 

Each link, when added to the spanning tree, creates exactly one cycle called a \emph{fundamental loop}. Dually, each twig defines a \emph{fundamental cutset}---the minimal set of branches whose removal disconnects the graph. These structures enable construction of two matrices that encode Kirchhoff's laws~\citep{kirchhoff1845ueber}.

\begin{definition}[Fundamental Loop Matrix, \citep{chua1987linear}]
  \label{def:loopmatrix}
  The fundamental loop matrix $\mathbf{B} \in \mathbb{Z}^{\ell \times b}$ encodes Kirchhoff's Voltage Law (KVL), with entries $B_{i,j} \in \{-1, 0, +1\}$ indicating whether branch $j$ appears in fundamental loop $i$ and with what orientation. For two-phase converters, we stack the matrices from both phases: $\mathbf{B} = [\mathbf{B}^1; \mathbf{B}^2]$. Partitioning by columns as $\mathbf{B} = [\mathbf{b}_{in} \mid \mathbf{B}_c]$ separates the input source column $\mathbf{b}_{in}$ from the submatrix $\mathbf{B}_c$ for capacitors and output.
\end{definition}

\begin{definition}[Fundamental Cutset Matrix, \citep{chua1987linear}]
  \label{def:cutsetmatrix}
  The fundamental cutset matrix $\mathbf{Q} \in \mathbb{Z}^{(n-1) \times b}$ (derived in Appendix~\ref{sec:Q_construction}) encodes Kirchhoff's Current Law (KCL), with entries $Q_{i,j} \in \{-1, 0, +1\}$ indicating whether branch $j$ appears in cutset $i$. Partitioning as $\mathbf{Q} = [\mathbf{Q}_c \mid \mathbf{q}_{out}]$ separates the submatrix $\mathbf{Q}_c$ for input and capacitors from the output column $\mathbf{q}_{out}$. 
\end{definition}

Not every circuit structure yields a functional converter---some have ambiguous voltages or charge flows. The submatrices $\mathbf{B}_c$ and $\mathbf{Q}_c$ determine whether a topology is valid.

\begin{definition}[Properly Posed, \citep{seeman2009design}]
  \label{def:proper}
  A two-phase SC converter is \emph{properly posed} if its topology uniquely determines all capacitor voltages and charge flows.
\end{definition}

\begin{theorem}[Properly Posed Criterion, \citep{seeman2009design}]
  \label{thm:proper}
  A two-phase converter is properly posed if and only if both $\mathbf{B}_c$ and $\mathbf{Q}_c$ are square and invertible.
\end{theorem}


Once a topology is verified as properly posed, we extract the VCR. Since capacitor voltages are constant under Assumption~\ref{ass:noload}, KVL must hold in both phases simultaneously, imposing $\mathbf{B}\mathbf{v} = \mathbf{0}$ where $\mathbf{v} = [V_{IN}, v_{c_1}, \ldots, v_{c_k}, V_{OUT}]^\top$. Using the partitioned form, this becomes $\mathbf{b}_{in} V_{IN} + \mathbf{B}_c \mathbf{v}_c = \mathbf{0}$, where $\mathbf{v}_c = [v_{c_1}, \ldots, v_{c_k}, V_{OUT}]^\top$.

\begin{theorem}[VCR Extraction, \citep{makowski1995performance}]
  \label{thm:vcr}
  For a properly posed converter, the capacitor and output voltages are $\mathbf{v}_c = -\mathbf{B}_c^{-1} \mathbf{b}_{in} V_{IN}$. Since voltages scale linearly with $V_{IN}$, setting $V_{IN} = 1$ gives $V_{OUT} = M$ directly, where:
  \begin{equation}
    M = \frac{V_{OUT}}{V_{IN}} = \left[-\mathbf{B}_c^{-1} \mathbf{b}_{in}\right]_{\textit{last}}
  \end{equation}
\end{theorem}

\subsubsection{SSL Performance Metric}
Under Assumption~\ref{ass:limits}, SSL efficiency is limited by capacitor charging losses, which depend on how charge distributes among capacitors. This is characterized by the \emph{charge multiplier vector} $\mathbf{a}_c$, where each element $a_{c,i}$ represents the charge flow into capacitor $i$ normalized by output charge. The full multiplier vector $\mathbf{a}$ (containing charge multipliers $\mathbf{a}_c$) is computed by solving $\mathbf{Q}_c \mathbf{a} = -\mathbf{q}_{out}$.

\begin{theorem}[SSL Metric, \citep{seeman2008analysis}]
  \label{thm:ssl}
  For a properly posed two-phase converter with capacitor charge multipliers $\mathbf{a}_c$ and capacitor voltages $\mathbf{v}_c$ (normalized to $V_{IN}$), the SSL performance metric is:
  \begin{equation}
    M_{SSL} = \frac{2M^2}{\left(\sum_{i \in caps} |a_{c,i}| \cdot |v_{c,i}|\right)^2}
  \end{equation}
\end{theorem}

\subsubsection{FSL Performance Metric}
Under Assumption~\ref{ass:limits}, FSL efficiency is limited by switch conduction losses, which depend on switch charge multipliers $\mathbf{a}_r$ and blocking voltages $\mathbf{v}_r$. A switch conducts current in one phase and blocks voltage in the other. The switch charge multiplier $a_{r,i}$ is computed from phase-specific cutset matrices relating switch currents to capacitor currents; the blocking voltage $v_{r,i}$ is obtained from the fundamental loop matrix of the phase where the switch is open (see Appendix~\ref{sec:fsl_extraction} for the explicit extraction equations).

\begin{theorem}[FSL Metric, \citep{seeman2008analysis}]
  \label{thm:fsl}
  For a properly posed two-phase converter with switch charge multipliers $\mathbf{a}_r$ and blocking voltages $\mathbf{v}_r$ (normalized to $V_{IN}$), the FSL performance metric is:
  \begin{equation}
    M_{FSL} = \frac{M^2}{2\left(\sum_{i \in sw} |a_{r,i}| \cdot |v_{r,i}|\right)^2}
  \end{equation}
\end{theorem}

Both metrics assume optimal component sizing—each capacitor or switch sized proportional to its charge-voltage product—enabling topology comparison independent of actual component values. Higher metrics indicate better utilization efficiency. Using Theorems~\ref{thm:proper}--\ref{thm:fsl}, PowerGenie validates, extracts VCRs, and computes metrics via matrix operations, enabling simulation-free evaluation at scale.

\subsection{Evolutionary Finetuning}
Building on its efficient analytical framework, PowerGenie achieves performance-driven alignment through evolutionary finetuning. It first pretrains a GPT model~\citep{brown2020language} to generate circuit topologies as Eulerian circuits—traversals encoding the entire structure by visiting each directed edge exactly once~\citep{anonymous2024analoggenie}. This model then undergoes evolutionary finetuning, which improves generation quality through co-evolution of the model and its training distribution. Traditional RL methods~\citep{vijayaraghavan2025autocircuit, gao2025eva} struggle with circuit generation due to reward hacking: when the policy finds a high-reward circuit, gradient updates reinforce that pattern, causing mode collapse onto similar designs that may exploit metric shortcuts. Our evolutionary approach breaks this cycle through tournament selection, which maintains diversity by giving lower-fitness circuits nonzero survival probability, and global uniqueness filtering, which requires each new circuit to differ from all previous designs. Together, these forcing continuous exploration rather than collapse onto a single mode, preventing repeated exploitation.

\begin{algorithm}[t]
\caption{Evolutionary Finetuning}
\label{alg:evolutionary-finetuning}
\begin{algorithmic}[1]
\STATE \textbf{Input:} Pretrained model $\mathcal{M}$, initial population $\mathcal{P}_0$, generations $G$, elite ratio $\alpha$, tournament size $k$
\STATE $\mathcal{H} \gets \emptyset$ \hfill $\triangleright$ History of analytically validated graphs
\FOR{$g = 1$ to $G$}
    \STATE \textit{// Phase 1: Circuit-Level Selection}
    \STATE Group $\mathcal{P}_{g-1}$ by circuit ID; compute per-circuit fitness
    \STATE $\mathcal{S}_g \gets$ top $\alpha$ circuits $\cup$ $k$-way tournament on rest
    \STATE \textit{// Phase 2: Fitness-Weighted Finetuning}
    \STATE Construct training batches from $\mathcal{S}_g$, sampling with prob. $\propto (\text{fitness} - \min + \epsilon)^\beta$
    \STATE Finetune $\mathcal{M}$ on sampled batches via next-token prediction loss~\citep{radford2018improving}
    \STATE \textit{// Phase 3: Generation with Global Uniqueness}
    \STATE $\mathcal{C}_{\text{new}} \gets$ sample circuits from $\mathcal{M}$
    \FOR{each $c \in \mathcal{C}_{\text{new}}$}
        \IF{$c$ passes functional validation and is not isomorphic to any $h \in \mathcal{H}$}
            \STATE $\mathcal{H} \gets \mathcal{H} \cup \{c\}$
            \STATE Augment $c$ with multiple Eulerian circuits
            \STATE $\mathcal{P}_g \gets \mathcal{P}_{g-1} \cup \{c \text{ and augmentations}\}$
        \ENDIF
    \ENDFOR
\ENDFOR
\STATE \textbf{return} $\mathcal{M}$, $\mathcal{P}_G$
\end{algorithmic}
\end{algorithm}

The evolutionary loop consists of three phases: selection, finetuning, and generation. Selection combines two complementary mechanisms to balance exploitation and exploration. \emph{Elite selection} directly preserves the top-$\alpha$ performing circuits across generations, ensuring that proven high-quality designs are never lost. \emph{Tournament selection} fills the remaining selection slots through repeated $k$-way competitions: for each slot, $k$ circuits are randomly sampled from the current population, and the one with the highest fitness wins selection. This probabilistic approach gives lower-fitness circuits a nonzero chance of survival. Such diversity preservation is crucial, as seemingly mediocre circuits may contain structural motifs that recombine into superior designs in subsequent generations.

During finetuning, batch sampling is weighted by fitness (Appendix~\ref{sec:evo_config}), focusing learning on high-performing circuits while maintaining population coverage. The model is then updated using standard next-token prediction loss~\citep{radford2018improving} on these fitness-weighted batches, effectively distilling high-performing circuit patterns into the model's learned distribution. The finetuned model implicitly learns crossover and mutation: by training on diverse high-fitness circuits, it internalizes successful structural motifs and recombines them during generation, which replaces hand-designed variation operators of traditional genetic algorithms~\citep{holland1992genetic} with a learned generative process. 

After finetuning, the model generates new candidate circuits via temperature-scaled sampling. To ensure genuine novelty, each candidate undergoes analytical evaluation and graph isomorphism testing against all historically validated circuits; only topologically unique and functionally valid designs enter the population. For each accepted circuit, we generate multiple Eulerian circuit representations via depth-first search, providing diverse training signals from the same topology. This co-evolutionary loop provides stable optimization without RL's high-variance gradient estimates.

\section{Results}
\label{sec: Results}
\subsection{Experiment Setup}
\noindent\textbf{Datasets.} PowerGenie expands the dataset from~\citep{anonymous2024analoggenie} to 11,837 unique circuit topologies across 11+ analog circuit types. Among these, 3,824 are reconfigurable power converters, of which 432 are 8-mode converters targeting VCRs of 1/5, 1/4, 1/3, 1/2, 2/3, 3/4, 4/5, and 1, labeled with FoM for performance-driven finetuning.

\noindent\textbf{Training Setup.} PowerGenie uses a 6-layer GPT model with 6 heads and embedding dimension 384 (10M parameters). Pretraining runs 100K steps with AdamW~\citep{loshchilov2017decoupled} (learning rate $6 \times 10^{-4}$, cosine decay, batch size 128). Evolutionary finetuning runs 500 generations, selecting 70\% of circuits per generation (elite ratio $\alpha = 0.15$, tournament size $k = 3$) and generating 256 candidates at temperature 0.7. Experiments use a single L40s GPU.

\noindent\textbf{Baselines.} We compare against: (1) Preference alignment methods PPO~\citep{schulman2017proximal} and DPO~\citep{rafailov2023direct} (2) AI-based power converter generation methods LaMAGIC~\citep{chang2024lamagic, chang2025lamagic2}, AUTOCIRCUIT-RL~\citep{vijayaraghavan2025autocircuit}, and AnalogGenie~\citep{anonymous2024analoggenie}. We also compare PowerGenie's best discovered circuit against the training set's best to evaluate whether it can surpass existing designs, using identical sizing optimization (Appendix~\ref{sec:sizing_opt}) to ensure fair comparison.

\noindent\textbf{Evaluation Metrics.} We sample 1,000 topologies from each method and measure pass rates across four criteria. \\
(1) \textit{Syntax Validity}: A topology is syntax-valid if it yields a well-formed SPICE netlist with proper device connections. \\
(2) \textit{Functional Validity}: A syntax-valid topology is functionally valid if it satisfies Theorem~\ref{thm:proper} and delivers all 8 target VCRs as verified by our analytical framework. \\
(3) \textit{Novelty}: A functionally valid topology is novel if it is not isomorphic to any topology in the training set and any topology discovered in previous generations. \\
(4) \textit{Figure of Merit}: For functionally valid topologies:
\begin{equation}
    \text{FoM} = \underbrace{\frac{1}{8}\sum_{m=1}^{8} \left(\frac{1}{2}\tilde{M}_{SSL}^{(m)} + \frac{1}{2}\tilde{M}_{FSL}^{(m)}\right)}_{\text{Efficiency}} - \underbrace{\tilde{N}_{cap}}_{\text{Area}}
    \label{eq:FoM}
\end{equation}
where $\tilde{M}_{SSL}^{(m)}$, $\tilde{M}_{FSL}^{(m)}$, and $\tilde{N}_{cap}$ are min-max normalized: $\tilde{x} = (x - x_{\min})/(x_{\max} - x_{\min})$. For SSL and FSL, $x_{\min} = 0$ and $x_{\max}$ is set to theoretical limits~\citep{seeman2008analysis}; for capacitor count, $N_{\min} = 0$ and $N_{\max} = 24$ (maximum in the training set). The efficiency term equally weights SSL and FSL losses, while the area term penalizes capacitor count, which dominates silicon area in integrated SC converters (Appendix~\ref{sec:sizing}). Higher FoM indicates better efficiency with fewer capacitors, ranging from $-1$ to $1$; we report the best FoM among 1,000 generated topologies.

\begin{table}[ht]
\centering
\caption{Generation quality comparison between evolutionary finetuning and existing preference alignment methods.}
\resizebox{\columnwidth}{!}{
\begin{threeparttable}
\begin{tabular}{lcccc}
\toprule
Method & Syntax $\uparrow$ & Functional $\uparrow$ & Novelty $\uparrow$ & FoM $\uparrow$ \\
\midrule
Pretrain & 38.2\% & 1.10\% & 0.4\% & 0.246 \\
Pretrain+SFT & 48.9\% & 40.2\% & 1.8\% & 0.263 \\
\midrule
Pretrain+DPO & 35.5\% & 1.6\% & 0.4\% & 0.256 \\
Pretrain+SFT+DPO & 44.5\% & 37.7\% & 2.7\% & 0.272 \\
\midrule
Pretrain+PPO & 29.7\% & 0.8\% & 0.3\% & 0.253 \\
Pretrain+SFT+PPO & 47.4\% & 39\% & 3.5\% & 0.284 \\
\midrule
\textbf{Pretrain+Evo} & \textbf{90.8\%} & \textbf{85.3\%} & \textbf{32.1\%} & \textbf{0.323} \\
\bottomrule
\end{tabular}
\begin{tablenotes}\scriptsize
\item All methods generate 1,000 topologies for evaluation. Training set best FoM = 0.263.
\end{tablenotes}
\end{threeparttable}
}
\label{tab:performance_comparison}
\end{table}

\begin{table}[ht]
\centering
\caption{Circuit complexity comparison between PowerGenie and existing AI-based power converter generation methods.}
\resizebox{\columnwidth}{!}{
\begin{threeparttable}
\begin{tabular}{lcccc}
\toprule
Method & \# Modes & \# Devices & \# Ctrl & Surpass Train\tnote{$\dagger$} \\
\midrule
LaMAGIC & 1 & 6 & 5 & No \\
AUTOCIRCUIT-RL & 1 & 10 & 5 & No \\
AnalogGenie & 1 & 7 & 2 & No \\
\midrule
\textbf{PowerGenie} & \textbf{8} & \textbf{57} & \textbf{65536} & \textbf{Yes} \\
\bottomrule
\end{tabular}
\begin{tablenotes}\scriptsize
\item[$\dagger$] Discovers circuit with FoM exceeding training set best.
\end{tablenotes}
\end{threeparttable}
}
\label{tab:complexity_comparison}
\end{table}

\subsection{Comparison with Preference Alignment Methods}
Table~\ref{tab:performance_comparison} presents a comprehensive comparison between evolutionary finetuning and existing preference alignment methods for 8-mode power converter topology generation. The pretrained model alone exhibits limited generation quality, with only 38.2\% syntax validity and 1.10\% functional validity, while standard supervised finetuning (SFT) provides meaningful improvements, raising syntax validity to 48.9\% and functional correctness to 40.2\%.

Notably, applying preference optimization methods directly to the pretrained model proves counterproductive---both DPO and PPO actually degrade syntax validity (35.5\% and 29.7\%, respectively) compared to the pretrained baseline, suggesting that without sufficient syntactic grounding from SFT, these methods optimize for preference signals at the expense of syntax validity. When combined with SFT, DPO and PPO recover and offer modest improvements in novelty (2.7\% and 3.5\%) and FoM (0.272 and 0.284), but the gains remain incremental---likely due to mode collapse where model repeatedly generates similar high-reward designs.

In stark contrast, evolutionary finetuning (Pretrain+Evo) achieves transformative improvements across all dimensions: 90.8\% syntax validity, 85.3\% functional validity, 32.1\% novelty, and an FoM of 0.323. The novelty rate is particularly striking---nearly an order of magnitude higher than the best preference-based method---validating that tournament selection and global uniqueness filtering successfully prevent mode collapse. Moreover, the achieved FoM of 0.323 substantially surpasses the training set's best of 0.263, demonstrating that the co-evolutionary loop discovers genuinely superior topologies beyond the training distribution.

\begin{table}[ht]
\centering
\caption{Ablation study of evolutionary finetuning components.}
\resizebox{\columnwidth}{!}{
\begin{threeparttable}
\begin{tabular}{lcccc}
\toprule
Method & Syntax $\uparrow$ & Functional $\uparrow$ & Novelty $\uparrow$ & FoM $\uparrow$ \\
\midrule
w/o Selection & 79.3\% & 73.7\% & 18.4\% & 0.261 \\
w/o Generation &  87.8\% & 82.6\% & 0.8\% &  0.263 \\
\midrule
\textbf{PowerGenie} & \textbf{90.8\%} & \textbf{85.3\%} & \textbf{32.1\%} & \textbf{0.323} \\
\bottomrule
\end{tabular}
\begin{tablenotes}\scriptsize
\item All methods generate 1,000 topologies for evaluation. Training set best FoM = 0.263.
\end{tablenotes}
\end{threeparttable}
}
\label{tab:ablation}
\end{table}

\subsection{Comparison with Existing AI-based Power Converter Design Methods}
Table~\ref{tab:complexity_comparison} compares PowerGenie against existing AI-based power converter generation methods. Prior approaches like LaMAGIC, AUTOCIRCUIT-RL, and AnalogGenie are limited to single-mode designs with modest device counts (6–10) and simple control requirements (2–5 states). In contrast, PowerGenie targets 8-mode reconfigurable converters, representing a qualitative leap in design complexity. The resulting circuits contain 57 devices with 65,536 possible control configurations ($2^{16}$ states across 8 modes and 2 phases), reflecting the exponential scaling of control complexity described above. Despite this dramatically expanded design space, PowerGenie discovers circuits with FoM exceeding the training set's best—a capability absent in all prior methods. This validates that the co-evolutionary loop successfully balances exploration and exploitation, enabling discovery of genuinely novel high-performance topologies rather than collapsing onto known designs. By scaling to the complexity regime demanded by modern processors and IoT applications, PowerGenie demonstrates that AI-driven topology discovery can address converter design challenges previously tractable only through manual expertise.

\subsection{Ablation Study}
Table~\ref{tab:ablation} validates that both selection and generation are essential components of evolutionary finetuning. Removing selection---which drives exploitation by preferentially sampling high-fitness circuits for training---reduces FoM from 0.323 to 0.261, falling below the training set's best of 0.263. While this variant maintains reasonable novelty (18.4\%), it generates novel circuits without ensuring quality, producing exploration without directed improvement. Conversely, removing generation---which drives exploration by producing new candidate circuits---yields FoM of exactly 0.263, matching but not exceeding the training best, with novelty collapsing to just 0.8\%. This variant achieves exploitation without exploration, effectively memorizing high-performing training examples rather than discovering new designs. Only when both components operate together does PowerGenie achieve the synergy necessary for genuine discovery: selection focuses learning on promising regions while generation continuously proposes novel candidates that push beyond known solutions.

\begin{figure*}[!t]
\begin{center}
\includegraphics[width=0.9\linewidth]{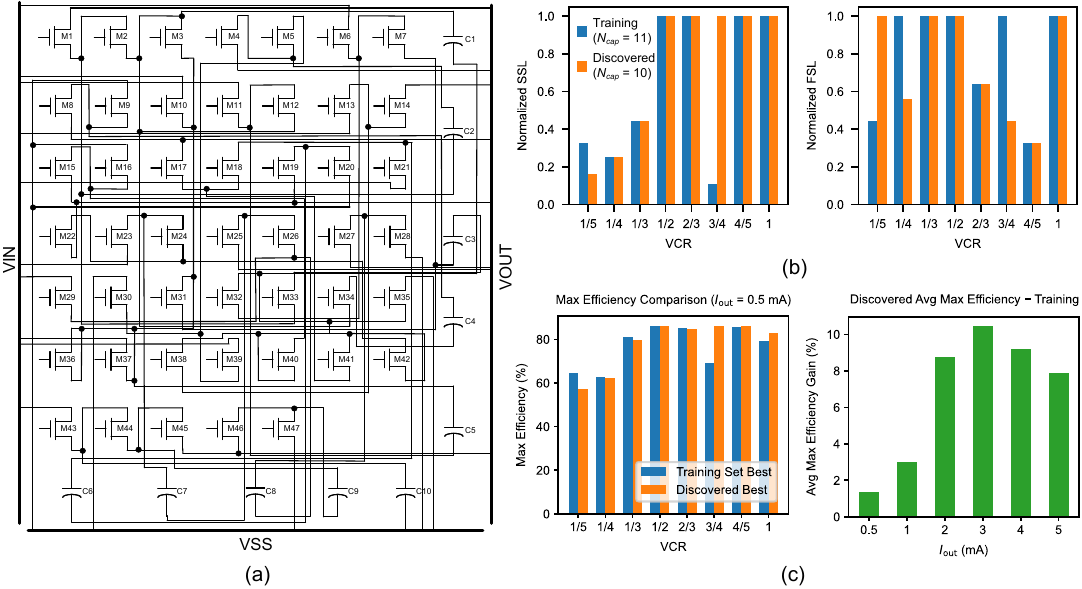}
\caption{(a) PowerGenie-discovered best eight-mode power converter. Discovered control scheme is shown in Table~\ref{tab:8mode_control_powergenie} (b) Analytical comparison: discovered-best achieves better SSL ($\tilde{M}_{\text{SSL}}$ = 0.732 vs.\ 0.642) with fewer capacitors ($\tilde{N}_{cap}$ = 0.417 vs.\ 0.458), while training-best shows better FSL ($\tilde{M}_{\text{FSL}}$ = 0.801 vs.\ 0.747). (c) SPICE results (TSMC 180\,nm, $f_{\text{sw}}$ = 10\,MHz, $C_{\text{total}}$ = 3\,nF, 2\% parasitic cap). Left: efficiency comparison at $I_{\text{out}}$ = 0.5\,mA aligns with SSL metrics, indicating capacitor loss dominance. Right: average absolute efficiency gain of discovered topologies increases with load (reaching $10.45\%$ at 3\,mA), demonstrating superior power handling capability.}
\label{fig: discover}
\end{center}
\end{figure*}

\subsection{Performance Validation}
To validate that evolutionary finetuning discovers genuinely superior topologies, we compare the best discovered eight-mode converter against the highest-performing topology in the training set. The analytical comparison in Figure~\ref{fig: discover}(b) reveals an interesting trade-off: the discovered circuit achieves substantially better SSL performance ($\tilde{M}_{\text{SSL}}=0.732$ vs.\ 0.642) while requiring fewer capacitors ($\tilde{N}_{cap}=0.417$ vs.\ 0.458), whereas the training-best topology exhibits superior FSL ($\tilde{M}_{\text{FSL}}=0.801$ vs.\ 0.747). These analytical advantages translate to measurable gains in SPICE simulation: as shown in Figure~\ref{fig: discover}(c) right, the average absolute efficiency gain of discovered topologies increases with load current, reaching $10.45\%$ at 3\,mA. This scaling behavior confirms that the discovered topology offers fundamentally superior power handling capability and efficiency, validating that evolutionary finetuning successfully explores beyond the design space boundaries defined by training circuits.

The SPICE simulation results (TSMC 180~nm, $f_{\text{sw}} = 10$~MHz, $C_{\text{total}} = 3$~nF, 2\% parasitic capacitance) in Figure~\ref{fig: discover}(c) left align with analytical SSL predictions, particularly at VCR~$= 1/5$ (Mode~1) and VCR~$= 3/4$ (Mode~6), where efficiency gaps closely mirror SSL differences. Such alignment suggests that capacitor-related losses dominate in the practical implementations, with switching conduction losses playing secondary roles. This insight further reinforces the importance of minimizing passive component count and optimizing charge transfer efficiency—objectives that PowerGenie's fitness function directly targets through SSL maximization and capacitor penalty terms. Crucially, while SPICE-based sizing optimization requires approximately two hours per topology, the analytical model evaluates performance limits within seconds, enabling rapid fitness assessment across thousands of candidates.
\section{Conclusion and Future Work}
\label{sec: conclusion}
This work presents PowerGenie, an evolutionary finetuning framework that enables the discovery of reconfigurable SC power converter topologies surpassing circuits in the training set. By treating circuit generation as an iterative refinement process guided by performance-aware fitness evaluation, PowerGenie demonstrates that generative models can effectively transcend the boundaries of their training data to explore novel design spaces. While the current framework focuses on SC converters, promising future extensions could include discovering of diverse converter architectures and generalizing to other switched-capacitor circuits, such as data converters and discrete-time filters, where similar charge redistribution principles govern performance. We believe PowerGenie establishes a solid foundation for AI-driven circuit discovery that can accelerate innovation across diverse analog and mixed-signal design domains.




\section*{Impact Statement}
This paper presents work whose goal is to advance the field of Machine
Learning. There are many potential societal consequences of our work, none
which we feel must be specifically highlighted here.


\bibliography{bibtex/example_paper}
\bibliographystyle{icml2026}

\newpage
\appendix
\onecolumn
\section*{Appendix Overview}
\label{sec:appendix_overview}

\begin{itemize}
    \item \hyperref[sec:appendix_theory]{\textbf{Appendix A: Theoretical Foundations}}
    \begin{itemize}
        \item \hyperref[sec:proof_properly_posed]{A.1 Proof of Theorem 3.6 (Properly Posed Criterion)}
        \item \hyperref[sec:proof_vcr]{A.2 Proof of Theorem 3.7 (VCR Extraction)}
        \item \hyperref[sec:proof_ssl]{A.3 Proof of Theorem 3.8 (SSL Metric Derivation)}
        \item \hyperref[sec:proof_fsl]{A.4 Proof of Theorem 3.9 (FSL Metric Derivation)}
    \end{itemize}
    
    \item \hyperref[sec:appendix_dataset]{\textbf{Appendix B: Circuit Representation and Dataset}}
    \begin{itemize}
        \item \hyperref[sec:graph_modeling]{B.1 Circuit Graph Modeling Details}
        \item \hyperref[sec:dataset_stats]{B.2 Dataset Statistics and Composition}
    \end{itemize}
    
    \item \hyperref[sec:appendix_power_implementation]{\textbf{Appendix C: PowerGenie Implementation Details}}
    \begin{itemize}
        \item \hyperref[sec:model_arch]{C.1 Model Architecture and Pretrain Training Hyperparameters}
        \item \hyperref[sec:evo_config]{C.2 Evolutionary Finetuning Configuration}
        \item \hyperref[sec:analytical_impl]{C.3 Analytical Framework Additional Details}
    \end{itemize}

    \item \hyperref[sec:appendix_base_implementation]{\textbf{Appendix D: Baselines Implementation Details}}
    \begin{itemize}
        \item \hyperref[sec:SFT]{D.1 Supervised Fine-Tuning (SFT)}
        \item \hyperref[sec:DPO]{D.2 Direct Preference Optimization (DPO)}
        \item \hyperref[sec:PPO]{D.3 Proximal Policy Optimization (PPO)}
    \end{itemize}

    \item \hyperref[sec:appendix_spice]{\textbf{Appendix E: SPICE Simulation Details}}
    \begin{itemize}
        \item \hyperref[sec:sizing_opt]{E.1 Sizing Optimization Methodology}
        \item \hyperref[sec:sim_setup]{E.2 Simulation Setup and Parameters}
        \item \hyperref[sec:time_complexity]{E.3 Analytical Framework vs. SPICE Computational Cost Comparison}
    \end{itemize}

    \item \hyperref[sec:appendix_discovered]{\textbf{Appendix F: Best Discovered Topology Additional Details}}
    \begin{itemize}
        \item \hyperref[sec:control_scheme]{F.1 Control Scheme}
        \item \hyperref[sec:sizing]{F.2 Sizing Parameter}
    \end{itemize}

    \item \hyperref[sec:appendix_training]{\textbf{Appendix G: Best Training Topology Details}}
    \begin{itemize}
        \item \hyperref[sec:topology_training]{G.1 Topology}
        \item \hyperref[sec:control_scheme_training]{G.2 Control Scheme}
        \item \hyperref[sec:sizing_training]{G.3 Sizing Parameter}
    \end{itemize}

    \item \hyperref[sec:appendix_extended]{\textbf{Appendix H: Additional Experimental Results}}
    \begin{itemize}
        \item \hyperref[sec:multirun]{H.1 Multi-Run Statistics}
        \item \hyperref[sec:fom_distribution]{H.2 FoM Distribution of Discovered Circuits}
    \end{itemize}

\end{itemize}

\newpage

\section{Theoretical Foundations}
\label{sec:appendix_theory}

This section provides formal proofs for the theorems stated in Section~\ref{sec: Approach}. The analysis follows the network-theoretic framework developed by \citet{seeman2009design} and \citet{makowski1995performance}.

\subsection{Proof of Theorem~\ref{thm:proper}}
\label{sec:proof_properly_posed}

\begin{theorem}[Properly Posed Criterion, \citep{seeman2009design}]
A two-phase SC converter is properly posed if and only if both $\mathbf{B}_c$ and $\mathbf{Q}_c$ are square and invertible.
\end{theorem}

\begin{proof}
We prove each direction by analyzing the linear systems arising from Kirchhoff's laws.

\textbf{Necessity.} Assume the converter is properly posed, i.e., all capacitor voltages and charge flows are uniquely determined. Under Assumption~\ref{ass:noload}, capacitor voltages remain constant across both phases, so KVL must hold simultaneously in both phase networks. Stacking the KVL equations from both phases:
\begin{equation}
    \mathbf{B}\mathbf{v} = \begin{bmatrix} \mathbf{B}^1 \\ \mathbf{B}^2 \end{bmatrix} \begin{bmatrix} V_{IN} \\ \mathbf{v}_c \end{bmatrix} = \mathbf{0}.
\end{equation}
Using the column partition $\mathbf{B} = [\mathbf{b}_{in} \mid \mathbf{B}_c]$ from Definition~\ref{def:loopmatrix}:
\begin{equation}
    \mathbf{B}_c \mathbf{v}_c = -\mathbf{b}_{in} V_{IN}.
    \label{eq:kvl_system}
\end{equation}

Let $\mathbf{B}_c \in \mathbb{R}^{m \times (k+1)}$ where $m$ is the number of KVL equations and $k+1 = \dim(\mathbf{v}_c)$. If $m < k+1$, the system is underdetermined and $\mathbf{v}_c$ is not unique. If $m > k+1$, the system is overdetermined and generically inconsistent. Since the converter is properly posed, we must have $m = k+1$, so $\mathbf{B}_c$ is square.

For square $\mathbf{B}_c$, if $\det(\mathbf{B}_c) = 0$, then $\ker(\mathbf{B}_c) \neq \{\mathbf{0}\}$. Any solution $\mathbf{v}_c^*$ would yield infinitely many solutions $\mathbf{v}_c^* + \mathbf{n}$ for $\mathbf{n} \in \ker(\mathbf{B}_c)$, contradicting uniqueness. Hence $\mathbf{B}_c$ is invertible.

An analogous argument on the KCL system $\mathbf{Q}_c \mathbf{a} = -\mathbf{q}_{out}$ (Definition~\ref{def:cutsetmatrix}) establishes that $\mathbf{Q}_c$ must also be square and invertible.

\textbf{Sufficiency.} Assume $\mathbf{B}_c$ and $\mathbf{Q}_c$ are square and invertible. Then Eq.~\ref{eq:kvl_system} has the unique solution $\mathbf{v}_c = -\mathbf{B}_c^{-1}\mathbf{b}_{in}V_{IN}$, and the KCL system has unique solution $\mathbf{a} = -\mathbf{Q}_c^{-1}\mathbf{q}_{out}$. All voltages and charge flows are uniquely determined, so the converter is properly posed.
\end{proof}

\subsection{Proof of Theorem~\ref{thm:vcr}}
\label{sec:proof_vcr}

\begin{theorem}[VCR Extraction, \citep{makowski1995performance}]
For a properly posed converter, the voltage conversion ratio is:
\begin{equation}
    M = \frac{V_{OUT}}{V_{IN}} = \left[-\mathbf{B}_c^{-1} \mathbf{b}_{in}\right]_{\textit{last}}
\end{equation}
\end{theorem}

\begin{proof}
Under Assumption~\ref{ass:noload}, the converter operates with open-circuited output and capacitor voltages constant across both phases. From the proof of Theorem~\ref{thm:proper}, the voltage vector satisfies:
\begin{equation}
    \mathbf{v}_c = -\mathbf{B}_c^{-1} \mathbf{b}_{in} V_{IN},
\end{equation}
where $\mathbf{v}_c = [v_{c_1}, \ldots, v_{c_k}, V_{OUT}]^\top$. Since $-\mathbf{B}_c^{-1} \mathbf{b}_{in}$ depends only on the topology (not on $V_{IN}$), all voltages scale linearly with $V_{IN}$. Extracting the last component:
\begin{equation}
    V_{OUT} = \left[-\mathbf{B}_c^{-1} \mathbf{b}_{in}\right]_{\textit{last}} \cdot V_{IN}.
\end{equation}
The conversion ratio $M = V_{OUT}/V_{IN}$ follows immediately.
\end{proof}

\subsection{Proof of Theorem~\ref{thm:ssl}}
\label{sec:proof_ssl}
We first establish two lemmas fundamental to the SSL analysis.

\begin{lemma}[Tellegen's Theorem for SC Networks, \citep{chua1987linear}]
\label{lem:tellegen}
For any network, any vector of branch voltages $\mathbf{v}$ satisfying KVL is orthogonal to any vector of branch charges $\mathbf{q}$ satisfying KCL: $\mathbf{v}^\top \mathbf{q} = 0$.
\end{lemma}

\begin{lemma}[Charge Conservation]
\label{lem:charge_conservation}
For each capacitor $i$ in a two-phase SC converter operating in periodic steady state, the charge multipliers satisfy $a_{c,i}^{(1)} + a_{c,i}^{(2)} = 0$, hence $|a_{c,i}^{(1)}| = |a_{c,i}^{(2)}| =: |a_{c,i}|$.
\end{lemma}

\begin{proof}
Under Assumption~\ref{ass:noload}, the converter operates in periodic steady state, so net charge on each capacitor over a complete cycle must be zero. Since $a_{c,i}^{(j)} := \Delta q_{c,i}^{(j)}/q_{out}$ denotes the normalized charge transfer during phase $j$, the constraint $a_{c,i}^{(1)} + a_{c,i}^{(2)} = 0$ follows immediately.
\end{proof}

\begin{theorem}[SSL Metric, \citep{seeman2008analysis}]
For a properly posed two-phase converter with capacitor charge multipliers $\mathbf{a}_c$ and capacitor voltages $\mathbf{v}_c$ (normalized to $V_{IN}$), the SSL performance metric is:
\begin{equation}
    M_{SSL} = \frac{2M^2}{\left(\sum_{i \in \textit{caps}} |a_{c,i}| \cdot |v_{c,i}|\right)^2}
\end{equation}
\end{theorem}

\begin{proof}
The proof proceeds in three steps: deriving the SSL output impedance, establishing optimal capacitor sizing, and extracting the topology-dependent metric.

\textbf{Step 1: SSL Output Impedance via Tellegen's Theorem.}

Under Assumption~\ref{ass:limits}, switch resistances are negligible and charge transfers are impulsive. We derive the output impedance by analyzing power balance using Tellegen's theorem. Consider the converter with input short-circuited ($V_{IN} = 0$) and output connected to an independent voltage source $V_{OUT}$.

By Lemma~\ref{lem:tellegen}, in each phase $j$, the branch voltage vector $\mathbf{v}^{(j)}$ and charge flow vector $\mathbf{q}^{(j)}$ satisfy $(\mathbf{v}^{(j)})^\top \mathbf{q}^{(j)} = 0$. Summing over both phases represents the total energy balance:
\begin{equation}
    \sum_{j=1}^{2} (\mathbf{v}^{(j)})^\top \mathbf{q}^{(j)} = 0.
\end{equation}
Substituting charge multipliers $\mathbf{q}^{(j)} = \mathbf{a}^{(j)} q_{out}$ and separating the output source term from the capacitor terms:
\begin{equation}
    V_{OUT} \cdot (a_{out}^{(1)} + a_{out}^{(2)}) \cdot q_{out} + \sum_{i \in \textit{caps}} \sum_{j=1}^{2} v_{c,i}^{(j)} \cdot a_{c,i}^{(j)} q_{out} = 0.
\end{equation}
Since the total normalized charge delivered to the output is unity ($a_{out}^{(1)} + a_{out}^{(2)} = 1$), the first term simplifies to $V_{OUT} \cdot q_{out}$. To obtain a dimensionally consistent voltage equation, we divide the entire expression by $q_{out}$:
\begin{equation}
    V_{OUT} + \sum_{i \in \textit{caps}} \sum_{j=1}^{2} a_{c,i}^{(j)} v_{c,i}^{(j)} = 0.
    \label{eq:tellegen_expanded}
\end{equation}

To simplify the capacitor terms, write $v_{c,i}^{(j)} = v_{c,i}^{(1)} + \Delta v_{c,i}^{(j)}$, where $\Delta v_{c,i}^{(j)}$ denotes the voltage change from phase 1. Substituting this into Eq.~\ref{eq:tellegen_expanded}:
\begin{equation}
    V_{OUT} + \sum_{i \in \textit{caps}} v_{c,i}^{(1)} \underbrace{\sum_{j=1}^{2} a_{c,i}^{(j)}}_{= 0} + \sum_{i \in \textit{caps}} \sum_{j=1}^{2} a_{c,i}^{(j)} \Delta v_{c,i}^{(j)} = 0.
\end{equation}
The middle term vanishes by charge conservation (Lemma~\ref{lem:charge_conservation}), leaving:
\begin{equation}
    V_{OUT} = -\sum_{i \in \textit{caps}} \sum_{j=1}^{2} a_{c,i}^{(j)} \Delta v_{c,i}^{(j)}.
    \label{eq:vout_delta}
\end{equation}

For a two-phase converter, the cumulative voltage change on capacitor $i$ at phase 2 is determined by the charge delivered divided by capacitance:
\begin{equation}
    \Delta v_{c,i}^{(2)} = \frac{a_{c,i}^{(2)} q_{out}}{C_i}.
\end{equation}
Substituting this back into Eq.~\ref{eq:vout_delta} reintroduces $q_{out}$ correctly through the constitutive capacitor relationship:
\begin{equation}
    V_{OUT} = -\sum_{i \in \textit{caps}} \frac{(a_{c,i}^{(2)})^2}{C_i} q_{out}.
    \label{eq:vout_qout}
\end{equation}

Finally, the voltage drop $|V_{OUT}|$ arises from output current $i_{out}$ flowing through an effective resistance. Using $q_{out} = i_{out}/f_{sw}$:
\begin{equation}
    R_{SSL} = \sum_{i \in \textit{caps}} \frac{a_{c,i}^2}{C_i f_{sw}}.
    \label{eq:rssl_general}
\end{equation}

\textbf{Step 2: Optimal Capacitor Sizing.}

To enable topology comparison independent of absolute component values, we derive the minimum $R_{SSL}$ under optimal sizing. For capacitors, the relevant cost metric is energy storage $E_{c,i} = \frac{1}{2} C_i v_{c,i}^2$. Minimizing Eq.~\ref{eq:rssl_general} subject to fixed total energy storage yields:
\begin{equation}
    C_i^* = \frac{2 E_{total}}{|v_{c,i}|} \cdot \frac{|a_{c,i}|}{\sum_k |a_{c,k}| \cdot |v_{c,k}|}.
    \label{eq:optimal_cap_ssl}
\end{equation}

\textbf{Step 3: SSL Performance Metric.}

Substituting Eq.~\ref{eq:optimal_cap_ssl} into Eq.~\ref{eq:rssl_general} yields the optimized SSL output impedance:
\begin{equation}
    R_{SSL}^* = \frac{\left(\sum_{i \in \textit{caps}} |a_{c,i}| \cdot |v_{c,i}|\right)^2}{2 E_{total} f_{sw}}.
    \label{eq:rssl_opt}
\end{equation}
Following~\citet{seeman2008analysis}, we define a topology-comparison metric as the ratio of the converter's power-handling capability to its capacitor resource cost:
\begin{equation}
    M_{SSL} := \frac{V_{OUT}^2 / R_{SSL}^*}{E_{total} \cdot f_{sw}}.
\end{equation}
Adopting the standard normalization convention $V_{IN} = 1$ (consistent with Theorem~\ref{thm:vcr}), capacitor voltages $v_{c,i}$ are expressed relative to the input voltage, and $V_{OUT} = M$ directly. Substituting yields:
\begin{equation}
    M_{SSL} = \frac{2M^2}{\left(\sum_{i \in \textit{caps}} |a_{c,i}| \cdot |v_{c,i}|\right)^2}.
\end{equation}
\end{proof}

\subsection{Proof of Theorem~\ref{thm:fsl}}
\label{sec:proof_fsl}

\begin{lemma}[Switch Current]
\label{lem:switch_current}
For a two-phase converter with duty cycle $D = 1/2$, the current through switch $i$ when conducting is $i_{r,i} = 2 a_{r,i} \cdot i_{out}$, where $a_{r,i} := q_{r,i}/q_{out}$ is the switch charge multiplier.
\end{lemma}

\begin{proof}
Switch $i$ conducts charge $q_{r,i} = a_{r,i} q_{out}$ during on-time $D/f_{sw}$. With $q_{out} = i_{out}/f_{sw}$:
\begin{equation}
    i_{r,i} = \frac{q_{r,i}}{D/f_{sw}} = \frac{a_{r,i} \cdot i_{out}/f_{sw}}{1/(2f_{sw})} = 2a_{r,i} i_{out}.
\end{equation}
\end{proof}

\begin{theorem}[FSL Metric, \citep{seeman2008analysis}]
For a properly posed two-phase converter with switch charge multipliers $\mathbf{a}_r$ and blocking voltages $\mathbf{v}_r$ (normalized to $V_{IN}$), the FSL performance metric is:
\begin{equation}
    M_{FSL} = \frac{M^2}{2\left(\sum_{i \in \textit{sw}} |a_{r,i}| \cdot |v_{r,i}|\right)^2}
\end{equation}
\end{theorem}

\begin{proof}
The proof proceeds in three steps, paralleling that of Theorem~\ref{thm:ssl}.

\textbf{Step 1: FSL Output Impedance.}

Switch charge multipliers $a_{r,i}$ are determined from capacitor charge multipliers via KCL at switch nodes. Under Assumption~\ref{ass:limits}, the average power dissipated in switch $i$ with on-resistance $R_i$ is:
\begin{equation}
    P_{r,i} = D \cdot R_i \cdot i_{r,i}^2 = \frac{1}{2} R_i (2a_{r,i} i_{out})^2 = 2R_i a_{r,i}^2 i_{out}^2.
\end{equation}

Equating total switch loss to the apparent loss through the output impedance $\sum P_{r,i} = R_{FSL} \cdot i_{out}^2$, we get:
\begin{equation}
    R_{FSL} = \sum_{i \in \textit{sw}} 2R_i a_{r,i}^2 = \sum_{i \in \textit{sw}} \frac{2a_{r,i}^2}{G_i},
    \label{eq:rfsl_general}
\end{equation}
where $G_i = 1/R_i$ is the switch conductance.

\textbf{Step 2: Optimal Switch Sizing.}

To enable topology comparison independent of absolute component values, we derive the minimum $R_{FSL}$ under optimal sizing. For switches, the cost metric is the G-V$^2$ product $X_{r,i} = G_i v_{r,i}^2$. Minimizing Eq.~\ref{eq:rfsl_general} subject to fixed total cost yields:
\begin{equation}
    G_i^* = \frac{X_{total}}{|v_{r,i}|} \cdot \frac{|a_{r,i}|}{\sum_k |a_{r,k}| \cdot |v_{r,k}|}.
    \label{eq:optimal_switch_fsl}
\end{equation}

\textbf{Step 3: FSL Performance Metric.}

Substituting Eq.~\ref{eq:optimal_switch_fsl} into Eq.~\ref{eq:rfsl_general} yields the optimized FSL output impedance:
\begin{equation}
    R_{FSL}^* = \frac{2\left(\sum_{i \in \textit{sw}} |a_{r,i}| \cdot |v_{r,i}|\right)^2}{X_{total}}.
    \label{eq:rfsl_opt}
\end{equation}
We define the topology-comparison metric as the ratio of the converter's power-handling capability to its switch resource cost:
\begin{equation}
    M_{FSL} := \frac{V_{OUT}^2 / R_{FSL}^*}{X_{total}}.
\end{equation}
Adopting the standard normalization convention $V_{IN} = 1$ (consistent with Theorem~\ref{thm:vcr}), switch blocking voltages $v_{r,i}$ are expressed relative to the input voltage, and $V_{OUT} = M$ directly. Substituting yields:
\begin{equation}
    M_{FSL} = \frac{M^2}{2\left(\sum_{i \in \textit{sw}} |a_{r,i}| \cdot |v_{r,i}|\right)^2}.
\end{equation}
\end{proof}

\section{Circuit Representation and Dataset}
\label{sec:appendix_dataset}
\subsection{Circuit Graph Modeling Details}
\label{sec:graph_modeling}

PowerGenie builds upon the graph modeling foundations established by AnalogGenie~\citep{anonymous2024analoggenie, gaoanaloggenie}, while introducing a novel control scheme representation specifically designed for reconfigurable power converters.

\paragraph{Device-Pin Graph Representation.}
We represent circuit topologies as device-pin graphs where each node corresponds to a specific device pin (e.g., \texttt{NM1\_D}, \texttt{NM1\_G}, \texttt{NM1\_S}, \texttt{NM1\_B} for an NMOS transistor). This pin-level representation ensures a unique one-to-one mapping between graphs and circuit topologies, eliminating the ambiguity present in device-level representations where edge-to-pin assignments remain undefined.

\paragraph{Eulerian Circuit Sequentialization.}
To enable autoregressive generation, we convert circuit graphs into sequences using Eulerian circuits. Given a finite connected undirected graph $\mathcal{G} = (V, E)$, we construct a directed graph $\mathcal{D} = (V, A)$ by replacing each undirected edge $\{u, v\} \in E$ with directed arcs. We then apply depth-first search to find the shortest closed path visiting every directed edge exactly once.

\paragraph{Compact Control Scheme Representation.}
The primary graph modeling innovation in PowerGenie addresses the combinatorial complexity of control schemes in reconfigurable converters. An 8-mode, two-phase converter requires 16-bit control schemes per switch---specifying on/off states across all mode-phase combinations. Na\"ively encoding all $2^{16} = 65{,}536$ possible control configurations as distinct tokens would create prohibitive tokenizer overhead and severely limit generalization. Instead, PowerGenie decomposes control complexity into graph connectivity. We define only 16 control tokens, \texttt{VCONT1} through \texttt{VCONT16}, each representing a single bit position in the control scheme. A switch's control scheme is then represented by connecting the switch node to only the \texttt{VCONT} nodes corresponding to its active (logic-1) bits. For example, consider a switch with control scheme \texttt{0001010100101010} (bits indexed 1--16 from right). This scheme has logic-1 at positions 2, 4, 6, 8, 10, 12, and 15. In the graph representation, we connect the switch to nodes \texttt{VCONT2}, \texttt{VCONT4}, \texttt{VCONT6}, \texttt{VCONT8}, \texttt{VCONT10}, \texttt{VCONT12}, and \texttt{VCONT15}. Figure~\ref{fig:control_signal_encoding} illustrates this encoding scheme for a representative switch in an 8-mode converter.

\begin{figure}[H]
    \centering
    \includegraphics[width=0.6\linewidth]{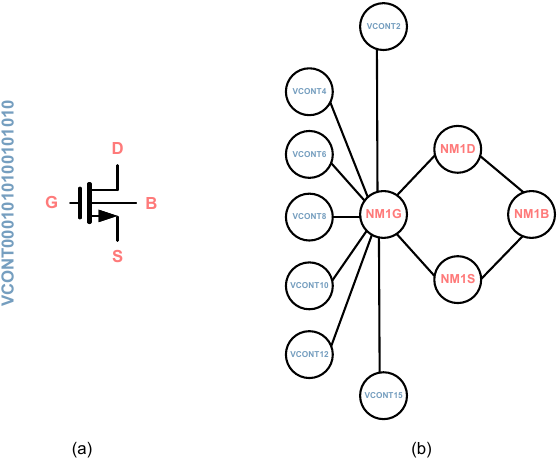}
     \caption{(a) A switch (NMOS transistor NM1) with 16-bit control scheme \texttt{0001010100101010} specifying on/off states across 8 modes and 2 phases. (b) Graph representation: device pins (NM1D, NM1G, NM1S, NM1B) form a cycle, and the gate node connects to \texttt{VCONT} nodes at active bit positions (2, 4, 6, 8, 10, 12, 15). This decomposition reduces tokenizer complexity from $2^{16} = 65{,}536$ to just 16 tokens while preserving full control scheme expressiveness.}
    \label{fig:control_signal_encoding}
\end{figure}

\subsection{Dataset Statistics and Composition}
\label{sec:dataset_stats}

PowerGenie expands the dataset from~\citep{anonymous2024analoggenie} to 11,837 unique circuit topologies across 11+ analog circuit types. Among these, 3,824 are reconfigurable power converters, of which 432 are 8-mode converters targeting VCRs of 1/5, 1/4, 1/3, 1/2, 2/3, 3/4, 4/5, and 1, labeled with FoM for performance-driven finetuning.

\paragraph{Device Count Distribution.}
Figure~\ref{fig:device_distribution} shows the distribution of device counts across all circuit topologies. The majority of circuits contain 21--30 devices (4,641 circuits), with substantial representation in the 11--20 (1,772) and 31--40 (1,639) device ranges. The dataset includes circuits ranging from simple designs with fewer than 10 devices to complex topologies with up to 90 devices, providing diverse training examples for scalable generation.

\paragraph{8-Mode Converter FoM Distribution.}
Figure~\ref{fig:fom_distribution} presents the FoM distribution for the 432 8-mode reconfigurable power converters used in evolutionary finetuning. The distribution has mean 0.074 and standard deviation 0.124, with FoM values ranging from $-0.351$ to 0.263. The best topology in the training set achieves FoM = 0.263, which serves as the baseline for evaluating whether PowerGenie can discover superior designs. As reported in Section~\ref{sec: Results}, evolutionary finetuning discovers topologies with FoM = 0.323, substantially exceeding this training set maximum.

\begin{figure}[H]
    \centering
    \includegraphics[width=0.6\linewidth]{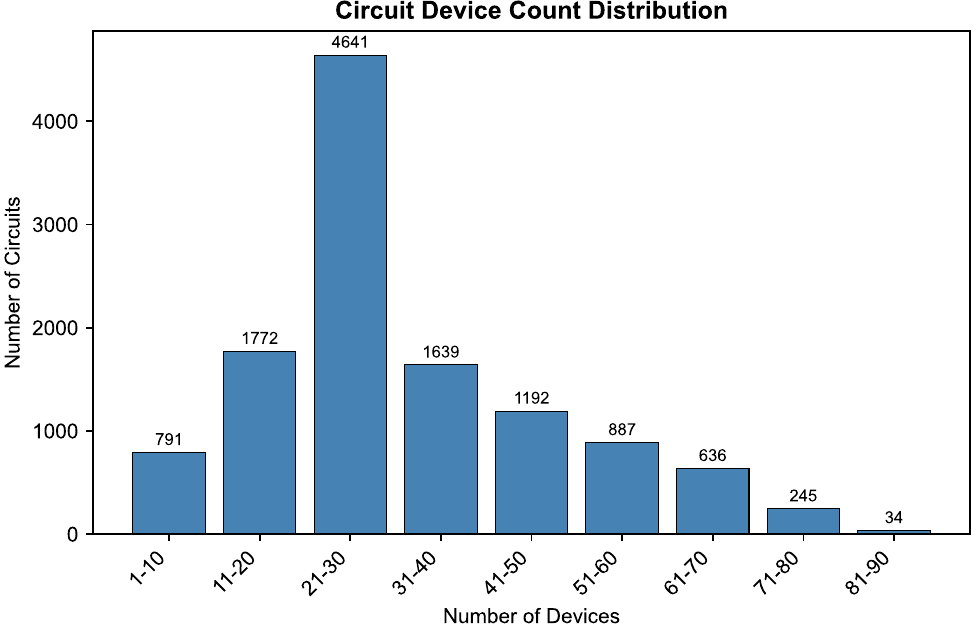}
    \caption{Distribution of device counts across all 11,837 circuit topologies. Most circuits contain 21--30 devices, with the dataset spanning from simple ($<$10 devices) to complex ($>$80 devices) designs.}
    \label{fig:device_distribution}
\end{figure}

\begin{figure}[H]
    \centering
    \includegraphics[width=0.6\linewidth]{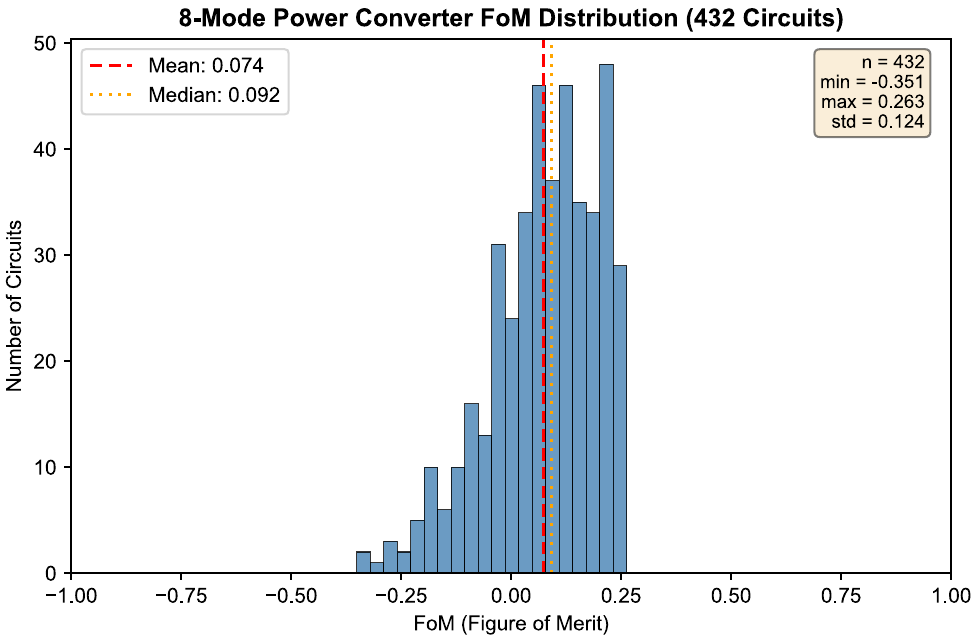}
    \caption{FoM distribution for the 432 8-mode reconfigurable power converters.}
    \label{fig:fom_distribution}
\end{figure}

\section{PowerGenie Implementation Details}
\label{sec:appendix_power_implementation}
\subsection{Model Architecture and Pretraining Hyperparameters}
\label{sec:model_arch}

\paragraph{Model Architecture.}
PowerGenie uses a decoder-only GPT architecture for autoregressive circuit topology generation. Table~\ref{tab:model_arch} summarizes the architectural configuration. The model employs pre-normalization (LayerNorm before attention and feed-forward layers), GELU activations, and weight tying between the token embedding and output projection layers. Flash attention is enabled when available for efficient training. Special scaled initialization is applied to residual projection weights, scaling by $1/\sqrt{2L}$ where $L$ is the number of layers.

\begin{table}[H]
\centering
\caption{GPT model architecture configuration.}
\label{tab:model_arch}
\begin{tabular}{lc}
\toprule
\textbf{Hyperparameter} & \textbf{Value} \\
\midrule
Number of layers & 6 \\
Number of attention heads & 6 \\
Embedding dimension & 384 \\
Feed-forward dimension & 1536 (4$\times$ embedding) \\
Maximum sequence length & 2048 \\
Vocabulary size & 1029 \\
Dropout rate & 0.2 \\
Bias in linear layers & False \\
Total parameters & $\sim$10M \\
\bottomrule
\end{tabular}
\end{table}

\paragraph{Tokenizer Design.}
The tokenizer maps device pins and circuit nodes to integer indices. Table~\ref{tab:tokenizer} summarizes the token categories. The vocabulary supports circuits with up to 69 NMOS/PMOS transistors, 26 NPN/PNP BJTs, 27 resistors, 29 capacitors, 23 inductors, and various other components. Control scheme tokens (\texttt{VCONT1}--\texttt{VCONT20}) enable compact representation of reconfigurable converter control schemes as described in Appendix~\ref{sec:graph_modeling}. The \texttt{TRUNCATE} token serves as padding for fixed-length sequence training.

\begin{table}[H]
\centering
\caption{Tokenizer vocabulary composition.}
\label{tab:tokenizer}
\begin{tabular}{lcc}
\toprule
\textbf{Token Category} & \textbf{Count} & \textbf{Pins per Device} \\
\midrule
NMOS transistors (NM1--NM69) & 276 & 4 (D, G, S, B) \\
PMOS transistors (PM1--PM69) & 276 & 4 (D, G, S, B) \\
NPN transistors (NPN1--NPN26) & 78 & 3 (C, B, E) \\
PNP transistors (PNP1--PNP26) & 78 & 3 (C, B, E) \\
Resistors (R1--R27) & 54 & 2 (P, N) \\
Capacitors (C1--C29) & 58 & 2 (P, N) \\
Inductors (L1--L23) & 46 & 2 (P, N) \\
Diodes (DIO1--DIO7) & 14 & 2 (P, N) \\
Control schemes (VCONT1--VCONT20) & 20 & -- \\
I/O nodes (VIN, VOUT, IIN, etc.) & 80+ & -- \\
Special tokens (VDD, VSS, TRUNCATE) & 3 & -- \\
\midrule
\textbf{Total vocabulary size} & \textbf{1029} & \\
\bottomrule
\end{tabular}
\end{table}

\paragraph{Pretraining Configuration.}
Table~\ref{tab:pretrain_config} details the pretraining hyperparameters. We use AdamW optimizer~\citep{loshchilov2017decoupled} with cosine learning rate decay and linear warmup. Mixed-precision training (bfloat16) and \texttt{torch.compile} are enabled for training efficiency. The model is pretrained on the full circuit dataset with 9:1 train/validation split. All experiments use a single NVIDIA L40S GPU.

\begin{table}[H]
\centering
\caption{Pretraining hyperparameters.}
\label{tab:pretrain_config}
\begin{tabular}{lc}
\toprule
\textbf{Hyperparameter} & \textbf{Value} \\
\midrule
Batch size & 128 \\
Maximum iterations & 100,000 \\
Learning rate (peak) & $6 \times 10^{-4}$ \\
Learning rate (minimum) & $6 \times 10^{-5}$ \\
Warmup iterations & 1,000 \\
LR decay iterations & 80,000 \\
Weight decay & $1 \times 10^{-2}$ \\
Adam $\beta_1$, $\beta_2$ & 0.9, 0.999 \\
Gradient clipping & 1.0 \\
Precision & bfloat16 \\
Torch compile & Enabled \\
\bottomrule
\end{tabular}
\end{table}

\subsection{Evolutionary Finetuning Configuration}
\label{sec:evo_config}

Evolutionary finetuning co-evolves the generative model with its training distribution through iterative selection, finetuning, and generation phases. This section details the hyperparameters governing each phase.

\paragraph{Evolutionary Loop Parameters.} Table~\ref{tab:evo_params} summarizes the core evolutionary algorithm configuration. The framework runs for 500 generations, generating 256 new candidate circuits per generation.

\begin{table}[H]
\centering
\caption{Evolutionary loop hyperparameters.}
\label{tab:evo_params}
\begin{tabular}{lc}
\toprule
\textbf{Hyperparameter} & \textbf{Value} \\
\midrule
Number of generations & 500 \\
Candidates generated per generation & 256 \\
Selection ratio & 0.7 \\
Elite ratio $\alpha$ & 0.15 \\
Tournament size $k$ & 3 \\
\bottomrule
\end{tabular}
\end{table}

\paragraph{Selection Strategy.} The selection phase operates at the circuit level rather than the sequence level, ensuring that selection pressure acts on unique topologies rather than their augmented representations. Given a population of circuits, selection proceeds as follows:

\begin{enumerate}
    \item \textbf{Circuit Grouping:} Sequences are grouped by circuit ID, with fitness computed per circuit (all augmented sequences of the same circuit share identical fitness).
    \item \textbf{Elite Selection:} The top $\alpha = 15\%$ of circuits by fitness are directly preserved.
    \item \textbf{Tournament Selection:} The remaining selection slots are filled through $k$-way tournaments ($k=3$): for each selection slot, $k$ circuits are randomly sampled and the highest-fitness circuit wins.
\end{enumerate}

\paragraph{Fitness-Weighted Batch Sampling.} During finetuning, batches are sampled with probability proportional to shifted fitness raised to a power:
\begin{equation}
    P(\text{sample } i) \propto \left( f_i - f_{\min} + \epsilon \right)^\beta,
\end{equation}
where $f_i$ is the fitness of example $i$, $f_{\min}$ is the minimum fitness in the population, $\epsilon = 10^{-6}$ prevents division by zero, and $\beta = 2.0$ controls the degree of fitness weighting.

\paragraph{Finetuning Hyperparameters.} Table~\ref{tab:finetune_params} details the optimizer configuration for finetuning.

\begin{table}[H]
\centering
\caption{Finetuning hyperparameters.}
\label{tab:finetune_params}
\begin{tabular}{lc}
\toprule
\textbf{Hyperparameter} & \textbf{Value} \\
\midrule
Batch size & 256 \\
Learning rate & $6 \times 10^{-5}$ \\
Weight decay & $5 \times 10^{-3}$ \\
Adam $\beta_1, \beta_2$ & 0.9, 0.999 \\
Gradient clipping & 0.5 \\
Dropout rate & 0.12 \\
Fitness weight power $\beta$ & 2.0 \\
\bottomrule
\end{tabular}
\end{table}

\paragraph{Generation Configuration.} Table~\ref{tab:gen_params} specifies the sampling parameters for candidate circuit generation. We use temperature-scaled sampling with $T = 0.7$ to balance diversity and quality.

\begin{table}[H]
\centering
\caption{Generation sampling parameters.}
\label{tab:gen_params}
\begin{tabular}{lc}
\toprule
\textbf{Hyperparameter} & \textbf{Value} \\
\midrule
Generation temperature & 0.7 \\
Maximum sequence length & 2,048 \\
Batch size per generation & 256 \\
\bottomrule
\end{tabular}
\end{table}

\paragraph{Data Augmentation.} Each unique circuit topology admits multiple valid Eulerian circuit representations. For each circuit, we generate up to 50 augmented sequences to provide diverse training signals.

\paragraph{Uniqueness Verification.} Generated circuits undergo graph isomorphism testing against all historically validated circuits across all previous generations. Only globally unique circuits are added to the population, preventing rediscovery of previously found topologies and ensuring continuous expansion with genuinely novel designs.

\subsection{Analytical Framework Additional Details}
\label{sec:analytical_impl}
This section details the mathematical derivation of the transformed fundamental cutset matrix $\mathbf{Q}$ and the explicit matrix equations used to extract switch parameters for the FSL metric.
\subsubsection{Construction and Transformation of the Cutset Matrix Q}
\label{sec:Q_construction}
This subsection provides the mathematical derivation of the transformed fundamental cutset matrix $\mathbf{Q}$ (referenced in Figure~\ref{fig: analytical} and Definition~\ref{def:cutsetmatrix}). While standard graph theory defines cutsets for a static graph, PowerGenie constructs a composite matrix to enforce Kirchhoff's Current Law (KCL) across multiple phases under the constraint of periodic steady-state charge conservation.

\paragraph{Variable Definition.}
We define the system charge vector $\mathbf{q}$ to represent the net charge transfer across cutsets over a full switching period $T$. For a two-phase converter, this vector is constructed by transforming the raw phase-specific charge flows. Let $\mathbf{q}_{raw}$ be the concatenation of charge flows in Phase 1 and Phase 2:
\begin{equation}
    \mathbf{q}_{raw} = [\Delta \mathbf{q}_{in}^{\phi 1}, \Delta \mathbf{q}_{c}^{\phi 1}, \Delta \mathbf{q}_{out}^{\phi 1} \mid \Delta \mathbf{q}_{in}^{\phi 2}, \Delta \mathbf{q}_{c}^{\phi 2}, \Delta \mathbf{q}_{out}^{\phi 2}]^\top
\end{equation}
where $\Delta \mathbf{q}_{c}^{\phi k}$ represents the vector of charges flowing into the capacitors during phase $k$.

\paragraph{Construction Steps.}
The construction of the matrix $\mathbf{Q}$ proceeds in three steps, following the methodology established in:

\begin{enumerate}
    \item \textbf{Phase-Specific KCL:} We first generate the standard fundamental cutset matrix for each phase topology independently, denoted as $\mathbf{Q}_{\phi 1}$ and $\mathbf{Q}_{\phi 2}$. These enforce instantaneous KCL in each phase network:
    \begin{equation}
        \mathbf{Q}_{\phi 1} \mathbf{q}^{\phi 1} = \mathbf{0}, \quad \mathbf{Q}_{\phi 2} \mathbf{q}^{\phi 2} = \mathbf{0}
    \end{equation}

    \item \textbf{Steady-State Constraints:} Under Assumption~\ref{ass:noload} (No-Load Steady State), the net charge accumulation on any flying capacitor over one period must be zero. This implies:
    \begin{equation}
        \Delta \mathbf{q}_{c, i}^{\phi 1} + \Delta \mathbf{q}_{c, i}^{\phi 2} = 0 \implies \Delta \mathbf{q}_{c, i}^{\phi 2} = -\Delta \mathbf{q}_{c, i}^{\phi 1}
    \end{equation}
    This constraint reduces the dimensionality of the problem. We designate $\Delta \mathbf{q}_{c}^{\phi 1}$ as the independent capacitor charge variable $\Delta \mathbf{q}_{c}$.

    \item \textbf{Matrix Stacking and Transformation:} We assemble the global system by stacking the phase matrices and applying a transformation matrix $\mathbf{T}_Q$ that maps the raw phase variables to the minimal set of independent variables $\mathbf{q}_{sys} = [\Delta \mathbf{q}_{in}^{\phi 1}, \Delta \mathbf{q}_{in}^{\phi 2}, \Delta \mathbf{q}_{c}, \Delta \mathbf{q}_{out}]$. 
\end{enumerate}

The transformed cutset matrix $\mathbf{Q}$ shown in Figure~\ref{fig: analytical} is derived as:
\begin{equation}
    \mathbf{Q} \cdot \mathbf{q}_{sys} = 
    \begin{bmatrix} 
        \mathbf{Q}_{\phi 1} & \mathbf{0} \\ 
        \mathbf{0} & \mathbf{Q}_{\phi 2} 
    \end{bmatrix} 
    \mathbf{T}_Q^{-1} \mathbf{q}_{sys} = \mathbf{0}
\end{equation}
where $\mathbf{T}_Q$ creates the linear combination of phase flows (e.g., mapping $\Delta \mathbf{q}_{c}^{\phi 2} \to -\Delta \mathbf{q}_{c}$). The resulting matrix $\mathbf{Q}$ is partitioned as $\mathbf{Q} = [\mathbf{Q}_c \mid \mathbf{q}_{out}]$, allowing us to solve for the charge multipliers $\mathbf{a}$ via $\mathbf{Q}_c \mathbf{a} = -\mathbf{q}_{out}$ as described in Section~\ref{sec: Approach}.

\subsubsection{Extraction of FSL Parameters}
\label{sec:fsl_extraction}

To compute the FSL metric (Theorem~\ref{thm:fsl}), PowerGenie automatically extracts the switch charge multipliers $\mathbf{a}_r$ and blocking voltages $\mathbf{v}_r$ for each phase $\phi \in \{1, 2\}$ using the phase-specific fundamental matrices.

\paragraph{Switch Blocking Voltages ($\mathbf{v}_r$).}
For a switch $i$ that is non-conducting (OFF) in phase $\phi$, its blocking voltage is determined by the KVL constraints of the phase network $G_{\phi}$. We treat the open switch as a \emph{link} that, when added to the spanning tree of capacitors and voltage sources, forms a fundamental loop. The fundamental loop matrix $\mathbf{B}^{\phi}$ encodes these dependencies.

By partitioning the rows of $\mathbf{B}^{\phi}$ to isolate the loops corresponding to the OFF-switches, and partitioning the columns into capacitor/input branches versus switch branches, the KVL equation becomes:
\begin{equation}
    \mathbf{v}_{sw}^{\phi} + \mathbf{B}_{c}^{\phi} \mathbf{v}_{c} + \mathbf{b}_{in}^{\phi} V_{IN} = \mathbf{0}
\end{equation}
where $\mathbf{v}_{sw}^{\phi}$ is the vector of blocking voltages for switches OFF in phase $\phi$. Solving explicitly yields:
\begin{equation}
    \mathbf{v}_{sw}^{\phi} = - (\mathbf{B}_{c}^{\phi} \mathbf{v}_{c} + \mathbf{b}_{in}^{\phi} V_{IN})
\end{equation}
Here, $\mathbf{v}_c$ are the steady-state capacitor voltages derived in Theorem~\ref{thm:proper}.

\paragraph{Switch Charge Multipliers ($\mathbf{a}_r$).}
For a switch $i$ that is conducting (ON) in phase $\phi$, its charge flow is determined by the KCL constraints. The conducting switch acts as a \emph{twig} (tree branch) in the spanning tree of $G_{\phi}$. Its fundamental cutset equation relates the switch current to the currents of the links (capacitors and output load) crossing that cutset.

Using the phase-specific fundamental cutset matrix $\mathbf{Q}^{\phi}$, we partition the rows to isolate the cutsets defined by the ON-switches:
\begin{equation}
    \mathbf{q}_{sw}^{\phi} + \mathbf{Q}_{c}^{\phi} \mathbf{q}_{c}^{\phi} + \mathbf{q}_{out}^{\phi} q_{out} = \mathbf{0}
\end{equation}
where $\mathbf{q}_{sw}^{\phi}$ represents the charge flows through switches ON in phase $\phi$. Solving explicitly yields:
\begin{equation}
    \mathbf{q}_{sw}^{\phi} = - (\mathbf{Q}_{c}^{\phi} \mathbf{q}_{c}^{\phi} + \mathbf{q}_{out}^{\phi} q_{out})
\end{equation}
The charge multipliers are then obtained by normalizing: $\mathbf{a}_{r}^{\phi} = \mathbf{q}_{sw}^{\phi} / q_{out}$. Note that $\mathbf{q}_{c}^{\phi}$ (capacitor charge flows in phase $\phi$) are known from the SSL solution (Theorem~\ref{thm:ssl}).

\section{Baseline Implementation Details}
\label{sec:appendix_base_implementation}

This section provides implementation details for the baseline methods compared in our experiments.

\subsection{Supervised Fine-Tuning (SFT)}
\label{sec:SFT}

Supervised fine-tuning serves as a baseline that trains directly on the labeled 8-mode converter dataset. The model is initialized from the pretrained checkpoint and fine-tuned using standard next-token prediction loss~\citep{radford2018improving} on the combined training and validation data.

\paragraph{Training Configuration.} Table~\ref{tab:sft_params} summarizes the SFT hyperparameters. We use a lower learning rate than pretraining to preserve learned circuit syntax while adapting to the target distribution.

\begin{table}[H]
\centering
\caption{SFT baseline hyperparameters.}
\label{tab:sft_params}
\begin{tabular}{lc}
\toprule
\textbf{Hyperparameter} & \textbf{Value} \\
\midrule
Maximum iterations & 2,000 \\
Batch size & 256 \\
Learning rate & $6 \times 10^{-5}$ \\
Weight decay & $5 \times 10^{-3}$ \\
Adam $\beta_1, \beta_2$ & 0.9, 0.999 \\
Gradient clipping & 0.5 \\
Dropout rate & 0.12 \\
\bottomrule
\end{tabular}
\end{table}

\subsection{Direct Preference Optimization (DPO)}
\label{sec:DPO}

Direct Preference Optimization~\citep{rafailov2023direct} aligns the model by directly optimizing on preference pairs without requiring a separate reward model. We initialize from the SFT or pretrain checkpoint and construct preference pairs by splitting the labeled dataset based on median FoM value.

\paragraph{Preference Pair Construction.} We partition the combined training and validation data into chosen and rejected sets based on the median label:
\begin{itemize}
    \item \textbf{Chosen:} Sequences with FoM $\geq$ median (higher-performing circuits)
    \item \textbf{Rejected:} Sequences with FoM $<$ median (lower-performing circuits)
\end{itemize}
This yields balanced preference pairs for training.

\paragraph{DPO Loss.} We define the implicit reward margin $u(x, y_w, y_l)$ as the difference in log-likelihood ratios between the winning and losing sequences:
\begin{equation}
    u(x, y_w, y_l) = \beta \log\frac{\pi_\theta(y_w|x)}{\pi_{\text{ref}}(y_w|x)} - \beta \log\frac{\pi_\theta(y_l|x)}{\pi_{\text{ref}}(y_l|x)}.
\end{equation}
The DPO objective with label smoothing is then given by:
\begin{equation}
    \mathcal{L}_{\text{DPO}} = -\mathbb{E}\left[(1-\epsilon)\log\sigma\left(u(x, y_w, y_l)\right) + \epsilon \log\sigma\left(-u(x, y_w, y_l)\right)\right],
\end{equation}
where $\pi_{\text{ref}}$ is the frozen SFT or pretrain model, $\beta$ controls deviation from the reference, and $\epsilon$ is the label smoothing factor.

\paragraph{Training Configuration.} Table~\ref{tab:dpo_params} summarizes the DPO hyperparameters. We use a significantly lower learning rate than SFT or pretrain to ensure stable preference learning without diverging from the reference policy.

\begin{table}[H]
\centering
\caption{DPO baseline hyperparameters.}
\label{tab:dpo_params}
\begin{tabular}{lc}
\toprule
\textbf{Hyperparameter} & \textbf{Value} \\
\midrule
Maximum iterations & 500 \\
Batch size & 256 \\
Learning rate & $1 \times 10^{-6}$ \\
Weight decay & $5 \times 10^{-3}$ \\
Adam $\beta_1, \beta_2$ & 0.9, 0.999 \\
Gradient clipping & 0.5 \\
Dropout rate & 0.12 \\
DPO $\beta$ (temperature) & 0.005 \\
Label smoothing $\epsilon$ & 0.1 \\
\bottomrule
\end{tabular}
\end{table}

\paragraph{Reference Model.} The reference model $\pi_{\text{ref}}$ is a frozen copy of the SFT or pretrain checkpoint, ensuring the policy does not diverge excessively during preference optimization.

\subsection{Proximal Policy Optimization (PPO)}
\label{sec:PPO}

Proximal Policy Optimization~\citep{schulman2017proximal} aligns the model through online reinforcement learning with a clipped surrogate objective. We initialize from the SFT or pretrain checkpoint and add a value head for advantage estimation.

\paragraph{Policy Architecture.} The policy network extends the SFT or pretrain model with a linear value head that predicts state values from the final hidden states. The value head is initialized with small weights (std $= 10^{-3}$) and zero bias to ensure stable initial value estimates.

\paragraph{Reference Model.} A frozen copy of the SFT checkpoint serves as the reference policy $\pi_{\text{ref}}$. We penalize divergence from this reference using the KL divergence:
\begin{equation}
    D_{\text{KL}}(\pi_\theta \| \pi_{\text{ref}}) = \sum_a \pi_\theta(a|s) \left(\log \pi_\theta(a|s) - \log \pi_{\text{ref}}(a|s)\right),
\end{equation}

\paragraph{Adaptive KL Coefficient.} The KL penalty coefficient $\beta$ adapts via smooth proportional control to maintain a target divergence:
\begin{equation}
    \beta_{t+1} = \beta_t \cdot \exp\left(\alpha \cdot (D_{\text{KL}} - D_{\text{target}})\right),
\end{equation}
where $\alpha$ controls adaptation speed. This increases $\beta$ when KL exceeds the target and decreases it otherwise, bounded by $[\beta_{\min}, \beta_{\max}]$.

\paragraph{Reward Structure.} Circuits receive shaped rewards based on a seven-stage validation cascade, providing dense feedback for policy learning.

\textbf{Stage 1: Syntax Validation (Tests 1--5).}
\begin{itemize}
    \item \textit{Test 1:} Floating pin detection---ensures each device pin connects to pins from other devices or ports
    \item \textit{Test 2:} Device completeness---verifies all required pins are present for each device instance
    \item \textit{Test 3:} Port connectivity---checks that non-control pins connect to at most one port
    \item \textit{Test 4:} Netlist parsing---validates that the sequence can be assembled into a well-formed SPICE netlist
    \item \textit{Test 5:} Graph conversion---constructs the circuit connectivity graph for analytical evaluation
\end{itemize}
Circuits failing any of tests 1--5 receive ladder rewards $r \in \{-3.5, -3.2, -2.9, -2.6, -2.3\}$ based on the number of first five tests passed. All 5 tests passed circuits will proceed to stage 2. 

\textbf{Stage 2: Functional Validation (Test 6).}
\begin{itemize}
    \item Properly posed criterion (Theorem~\ref{thm:proper})
    \item Correct voltage conversion ratio
\end{itemize}
Circuits with fewer than 8 valid modes receive ladder rewards: $r \in \{-2.0, -1.9, -1.8, -1.7, -1.6, -1.5, -1.4, -1.3\}$ for 0--7 valid modes respectively. Circuits achieving all 8 valid modes receive their FoM value with range of $[-1.0, 1.0]$.

\textbf{Stage 3: Novelty Verification (Test 7).} Graph isomorphism testing against the training database and previously validated circuits. Unique 8-mode circuits receive an additional bonus of $+0.01$.

\textbf{Completion Bonus.} All circuits sequence that terminate properly (emit the end-of-sequence TRUNCATE token) receive an additional bonus of $+0.01$, encouraging the policy to generate complete sequences.

\paragraph{Advantage Estimation.} We use Generalized Advantage Estimation (GAE) with $\gamma = 1.0$ and $\lambda = 0.95$:
\begin{equation}
    \hat{A}_t = \sum_{l=0}^{\infty} (\gamma \lambda)^l \delta_{t+l}, \quad \delta_t = r_t + \gamma V(s_{t+1}) - V(s_t),
\end{equation}
where rewards are assigned only at sequence termination. Advantages are normalized to zero mean and unit variance within each batch.

\paragraph{Training Configuration.} Table~\ref{tab:ppo_params} summarizes the PPO hyperparameters.
\begin{table}[H]
\centering
\caption{PPO baseline hyperparameters.}
\label{tab:ppo_params}
\begin{tabular}{lc}
\toprule
\textbf{Hyperparameter} & \textbf{Value} \\
\midrule
Training steps & 500 \\
Episodes per batch & 256 \\
Micro-batch size & 64 \\
PPO epochs per batch & 4 \\
Learning rate & $1 \times 10^{-6}$ \\
Gradient clipping & 0.5 \\
Weight decay & $5 \times 10^{-3}$ \\
Adam $\beta_1, \beta_2$ & 0.9, 0.999 \\
\midrule
\multicolumn{2}{c}{\textit{PPO-specific}} \\
\midrule
Clip $\epsilon$ & 0.2 \\
Value coefficient & 0.1 \\
Discount $\gamma$ & 1.0 \\
GAE $\lambda$ & 0.95 \\
Generation temperature & 0.7 \\
\midrule
\multicolumn{2}{c}{\textit{Adaptive KL}} \\
\midrule
Initial $\beta$ & 0.01 \\
Target KL & 0.1 \\
Adaptation rate $\alpha$ & 0.1 \\
$\beta$ bounds & $[10^{-4}, 10]$ \\
\bottomrule
\end{tabular}
\end{table}

\section{SPICE Simulation Details}
\label{sec:appendix_spice}

This section describes the SPICE simulation methodology used to validate the analytical predictions and compare discovered topologies against training set circuits.

\subsection{Sizing Optimization Methodology}
\label{sec:sizing_opt}

Validating topology performance requires sizing optimization to determine optimal component values, followed by transient simulation to measure actual efficiency. This workflow is applied identically to both discovered and baseline topologies

\paragraph{Mode-Independent Optimization with Reconciliation.} 

\begin{enumerate}
    \item \textbf{Per-mode optimization:} For each of the 8 VCR modes independently, we optimize switch sizes to maximize a weighted objective combining peak efficiency and output power. Since only a subset of switches conducts in any given mode, optimization is restricted to active switches, reducing problem dimensionality and accelerating convergence.
    
    \item \textbf{Cross-mode reconciliation:} After obtaining optimal sizes for each mode, we reconcile by selecting, for each switch, the maximum size across all modes:
    \begin{equation}
        W_i^* = \max_{m \in \{1,\ldots,8\}} W_i^{(m)},
    \end{equation}
    where $W_i^{(m)}$ is the optimal width for switch $i$ in mode $m$. This ensures each switch is sized to handle its worst-case current requirement across all operating modes.
\end{enumerate}

\paragraph{Optimization Algorithm.} Optimization is performed using the built-in \texttt{mvarsearch} function in Spectre MDL. The objective function is a weighted sum of peak efficiency and output power, balancing conversion efficiency against power delivery capability.

\subsection{Simulation Setup and Parameters}
\label{sec:sim_setup}

To evaluate converter performance, we simulate the generated circuit netlists using Cadence Spectre X with the TSMC 180nm process design kit.

\paragraph{Operating Conditions.} The converter is driven by a 1~V input supply. For each VCR mode, the output voltage $V_{\text{out}}$ is swept from 0 to $1 \times \text{VCR}$ to characterize efficiency across the full operating range. This sweep captures both light-load and heavy-load behavior, enabling comprehensive performance comparison.

\paragraph{Switch Driving.} NMOS switches are driven by ideal pulse voltage sources with a 0--2~V swing at 10~MHz switching frequency. A 48\% duty cycle is used for each phase to prevent shoot-through from overlapping conduction periods, ensuring a small dead-time margin between phase transitions.

\paragraph{Measurements.} For each operating point in the $V_{\text{out}}$ sweep, we measure:
\begin{itemize}
    \item \textbf{Output power:} $P_{\text{out}} = V_{\text{out}} \times I_{\text{out}}$
    \item \textbf{Conversion efficiency:} $\eta = P_{\text{out}} / P_{\text{in}}$
\end{itemize}
Peak efficiency and the corresponding output power are extracted for each VCR mode to enable cross-topology comparison.

\paragraph{Simulation Parameters.} Table~\ref{tab:sim_params} summarizes the simulation configuration.

\begin{table}[H]
\centering
\caption{SPICE simulation parameters.}
\label{tab:sim_params}
\begin{tabular}{lc}
\toprule
\textbf{Parameter} & \textbf{Value} \\
\midrule
Simulator & Cadence Spectre X \\
Process technology & TSMC 180nm \\
Input voltage $V_{\text{in}}$ & 1 V \\
Switching frequency $f_{\text{sw}}$ & 10 MHz \\
Phase duty cycle & 48\% \\
Gate drive swing & 0--2 V \\
Total flying capacitance $C_{\text{total}}$ & 3 nF \\
Parasitic capacitance ratio & 2\% \\
\bottomrule
\end{tabular}
\end{table}

\subsection{Analytical Framework vs. SPICE Computational Cost Comparison}
\label{sec:time_complexity}

A key advantage of the analytical framework is its computational efficiency compared to SPICE-based optimization, enabling large-scale evolutionary search that would otherwise be intractable.

\paragraph{Per-Topology Evaluation Time.} Table~\ref{tab:time_comparison} compares the evaluation time for a single 8-mode reconfigurable converter topology.

\begin{table}[H]
\centering
\caption{Per-topology evaluation time comparison.}
\label{tab:time_comparison}
\begin{tabular}{lcc}
\toprule
\textbf{Method} & \textbf{Time per Topology} & \textbf{Speedup} \\
\midrule
Analytical framework & 0.07 s & $125{,}000\times$ \\
SPICE optimization & 8,741 s ($\approx$2.4 h) & $1\times$ \\
\bottomrule
\end{tabular}
\end{table}

The analytical framework evaluates a topology in 0.07 seconds by directly computing SSL and FSL metrics through matrix operations (Theorems~\ref{thm:ssl}--\ref{thm:fsl}). In contrast, SPICE-based evaluation requires sizing optimization across all 8 modes, with each mode requiring approximately $500$ Spectre simulations for gradient-based search convergence. At 2.19 seconds per simulation, the total SPICE evaluation time is:
\begin{equation}
    T_{\text{SPICE}} = 8 \times 500 \times 2.19 \text{ s} = 8{,}741 \text{ s} \approx 2.4 \text{ hours}.
\end{equation}

\paragraph{Evolutionary Finetuning Scalability.} The computational advantage becomes decisive at the scale required for evolutionary finetuning. With 500 generations and 256 candidate topologies per generation, the total evaluation burden is:
\begin{equation}
    N_{\text{total}} = 500 \times 256 = 128{,}000 \text{ topologies}.
\end{equation}

Table~\ref{tab:evo_time} compares the total power converter evaluation time for the complete evolutionary finetuning process.

\begin{table}[H]
\centering
\caption{Total evaluation time for evolutionary finetuning (power converter functionality and performance evaluation only).}
\label{tab:evo_time}
\begin{tabular}{lc}
\toprule
\textbf{Method} & \textbf{Total Time} \\
\midrule
Analytical framework & 2.5 hours \\
SPICE optimization & 35.5 years \\
\bottomrule
\end{tabular}
\end{table}

The analytical framework completes all 128,000 evaluations in approximately 2.5 hours, making evolutionary finetuning practical on a single GPU. In contrast, SPICE-based evaluation would require over 35 years of computation time. This five-orders-of-magnitude speedup is what enables PowerGenie to explore the vast topology design space and discover circuits that surpass the training distribution.

\section{Best Discovered Topology Additional Details}
\label{sec:appendix_discovered}

This section provides the complete control scheme for the best 8-mode reconfigurable power converter discovered by PowerGenie.

\subsection{Control Scheme}
\label{sec:control_scheme}

Table~\ref{tab:8mode_control_powergenie} presents the complete control scheme for the discovered 47-switch, 10-capacitor, 8-mode reconfigurable power converter. Each entry specifies whether a switch is on (1) or off (0) during each clock phase ($\Phi_1$ or $\Phi_2$) for each voltage conversion ratio (VCR) mode.

\paragraph{Control Complexity.} Each switch requires a 16-bit control scheme specifying its on/off state across 8 modes $\times$ 2 phases, yielding $2^{16} = 65,536$ possible configurations per switch. With 47 switches, the full converter control space spans $2^{16 \times 47} = 2^{752}$ possible configurations—a search space far beyond manual exploration. PowerGenie's evolutionary finetuning automatically discovers both the topology and its corresponding control scheme within this vast design space.

\begin{table}[H]
\centering
\caption{Best discovered 8-mode power converter control scheme.}
\label{tab:8mode_control_powergenie}
\adjustbox{max width=\textwidth, max height=0.45\textheight}{
\begin{tabular}{|c|c|c|c|c|c|c|c|c|c|c|c|c|c|c|c|c|}
\hline
 & \multicolumn{2}{c|}{VCR=1} & \multicolumn{2}{c|}{VCR=4/5} & \multicolumn{2}{c|}{VCR=3/4} & \multicolumn{2}{c|}{VCR=2/3} & \multicolumn{2}{c|}{VCR=1/2} & \multicolumn{2}{c|}{VCR=1/3} & \multicolumn{2}{c|}{VCR=1/4} & \multicolumn{2}{c|}{VCR=1/5} \\
\hline
 & $\Phi$2 & $\Phi$1 & $\Phi$2 & $\Phi$1 & $\Phi$2 & $\Phi$1 & $\Phi$2 & $\Phi$1 & $\Phi$2 & $\Phi$1 & $\Phi$2 & $\Phi$1 & $\Phi$2 & $\Phi$1 & $\Phi$2 & $\Phi$1 \\
\hline
M1 & 0 & 0 & 0 & 1 & 0 & 1 & 0 & 1 & 0 & 0 & 0 & 1 & 1 & 0 & 0 & 1 \\
\hline
M2 & 0 & 0 & 1 & 0 & 1 & 0 & 1 & 0 & 0 & 0 & 1 & 0 & 0 & 1 & 1 & 0 \\
\hline
M3 & 0 & 0 & 0 & 1 & 0 & 1 & 0 & 1 & 0 & 0 & 0 & 1 & 1 & 0 & 0 & 1 \\
\hline
M4 & 0 & 0 & 1 & 0 & 1 & 0 & 1 & 0 & 0 & 0 & 1 & 0 & 0 & 1 & 1 & 0 \\
\hline
M5 & 0 & 0 & 0 & 1 & 0 & 1 & 0 & 1 & 0 & 0 & 0 & 1 & 1 & 0 & 0 & 1 \\
\hline
M6 & 0 & 0 & 0 & 0 & 0 & 0 & 0 & 0 & 0 & 0 & 0 & 0 & 0 & 0 & 0 & 1 \\
\hline
M7 & 0 & 0 & 0 & 0 & 0 & 0 & 0 & 0 & 0 & 0 & 0 & 0 & 0 & 0 & 1 & 0 \\
\hline
M8 & 0 & 0 & 0 & 0 & 0 & 0 & 0 & 0 & 0 & 0 & 0 & 0 & 0 & 0 & 0 & 1 \\
\hline
M9 & 0 & 0 & 0 & 0 & 0 & 0 & 0 & 0 & 0 & 0 & 0 & 0 & 0 & 0 & 1 & 0 \\
\hline
M10 & 0 & 0 & 0 & 0 & 0 & 0 & 0 & 0 & 0 & 0 & 0 & 0 & 0 & 1 & 0 & 0 \\
\hline
M11 & 0 & 0 & 0 & 0 & 0 & 0 & 0 & 0 & 0 & 0 & 0 & 0 & 1 & 0 & 0 & 0 \\
\hline
M12 & 0 & 0 & 0 & 0 & 0 & 0 & 0 & 0 & 0 & 0 & 0 & 0 & 0 & 1 & 0 & 0 \\
\hline
M13 & 0 & 0 & 0 & 0 & 0 & 0 & 0 & 0 & 0 & 0 & 1 & 0 & 0 & 0 & 0 & 0 \\
\hline
M14 & 0 & 0 & 0 & 0 & 0 & 0 & 0 & 0 & 0 & 0 & 0 & 1 & 0 & 0 & 0 & 0 \\
\hline
M15 & 0 & 0 & 0 & 0 & 0 & 0 & 0 & 1 & 0 & 0 & 0 & 0 & 0 & 0 & 0 & 0 \\
\hline
M16 & 0 & 0 & 0 & 0 & 0 & 0 & 1 & 0 & 0 & 0 & 0 & 0 & 0 & 0 & 0 & 0 \\
\hline
M17 & 0 & 0 & 0 & 0 & 0 & 1 & 0 & 0 & 0 & 0 & 0 & 0 & 0 & 0 & 0 & 0 \\
\hline
M18 & 0 & 0 & 0 & 0 & 1 & 0 & 0 & 0 & 0 & 0 & 0 & 0 & 0 & 0 & 0 & 0 \\
\hline
M19 & 0 & 0 & 0 & 0 & 0 & 1 & 0 & 0 & 0 & 0 & 0 & 0 & 0 & 0 & 0 & 0 \\
\hline
M20 & 0 & 0 & 0 & 0 & 1 & 0 & 0 & 0 & 0 & 0 & 0 & 0 & 0 & 0 & 0 & 0 \\
\hline
M21 & 0 & 0 & 0 & 0 & 0 & 1 & 0 & 0 & 0 & 0 & 0 & 0 & 0 & 0 & 0 & 0 \\
\hline
M22 & 0 & 0 & 0 & 1 & 0 & 0 & 0 & 0 & 0 & 0 & 0 & 0 & 0 & 0 & 0 & 0 \\
\hline
M23 & 0 & 0 & 0 & 1 & 0 & 0 & 0 & 0 & 0 & 0 & 0 & 0 & 0 & 0 & 0 & 0 \\
\hline
M24 & 0 & 0 & 1 & 0 & 0 & 0 & 0 & 0 & 0 & 0 & 0 & 0 & 0 & 0 & 0 & 0 \\
\hline
M25 & 0 & 0 & 0 & 1 & 0 & 0 & 0 & 0 & 0 & 0 & 0 & 0 & 0 & 0 & 0 & 0 \\
\hline
M26 & 0 & 0 & 1 & 0 & 0 & 0 & 0 & 0 & 0 & 0 & 0 & 0 & 0 & 0 & 0 & 0 \\
\hline
M27 & 0 & 0 & 0 & 1 & 0 & 0 & 0 & 0 & 0 & 0 & 0 & 0 & 0 & 0 & 0 & 0 \\
\hline
M28 & 0 & 0 & 1 & 0 & 0 & 0 & 0 & 0 & 0 & 0 & 0 & 0 & 0 & 0 & 0 & 0 \\
\hline
M29 & 0 & 0 & 0 & 1 & 0 & 0 & 0 & 0 & 0 & 0 & 0 & 0 & 0 & 0 & 0 & 0 \\
\hline
M30 & 0 & 0 & 0 & 0 & 0 & 0 & 0 & 0 & 0 & 0 & 1 & 1 & 1 & 1 & 1 & 1 \\
\hline
M31 & 0 & 0 & 0 & 0 & 0 & 0 & 0 & 0 & 0 & 0 & 0 & 0 & 0 & 0 & 1 & 1 \\
\hline
M32 & 0 & 0 & 0 & 0 & 0 & 0 & 0 & 0 & 0 & 0 & 0 & 0 & 0 & 0 & 1 & 1 \\
\hline
M33 & 0 & 0 & 0 & 0 & 0 & 0 & 0 & 0 & 0 & 0 & 0 & 0 & 0 & 0 & 1 & 1 \\
\hline
M34 & 0 & 0 & 0 & 0 & 0 & 0 & 0 & 0 & 0 & 0 & 0 & 0 & 0 & 0 & 1 & 1 \\
\hline
M35 & 0 & 0 & 0 & 0 & 0 & 0 & 0 & 0 & 0 & 0 & 0 & 0 & 1 & 1 & 0 & 0 \\
\hline
M36 & 0 & 0 & 0 & 0 & 0 & 0 & 0 & 0 & 0 & 0 & 0 & 0 & 1 & 1 & 0 & 0 \\
\hline
M37 & 0 & 0 & 1 & 1 & 1 & 1 & 1 & 1 & 0 & 0 & 1 & 1 & 0 & 0 & 0 & 0 \\
\hline
M38 & 0 & 0 & 0 & 0 & 0 & 0 & 0 & 0 & 0 & 0 & 1 & 1 & 0 & 0 & 0 & 0 \\
\hline
M39 & 0 & 0 & 1 & 1 & 1 & 1 & 1 & 1 & 0 & 0 & 0 & 0 & 0 & 0 & 0 & 0 \\
\hline
M40 & 0 & 0 & 0 & 0 & 0 & 0 & 1 & 1 & 0 & 0 & 0 & 0 & 0 & 0 & 0 & 0 \\
\hline
M41 & 0 & 0 & 0 & 0 & 1 & 1 & 0 & 0 & 0 & 0 & 0 & 0 & 0 & 0 & 0 & 0 \\
\hline
M42 & 0 & 0 & 1 & 1 & 0 & 0 & 0 & 0 & 0 & 0 & 0 & 0 & 0 & 0 & 0 & 0 \\
\hline
M43 & 0 & 1 & 0 & 0 & 0 & 0 & 0 & 0 & 1 & 1 & 0 & 0 & 0 & 0 & 0 & 0 \\
\hline
M44 & 0 & 1 & 0 & 0 & 0 & 0 & 0 & 0 & 0 & 1 & 0 & 0 & 0 & 0 & 0 & 0 \\
\hline
M45 & 1 & 0 & 0 & 0 & 0 & 0 & 0 & 0 & 1 & 0 & 0 & 0 & 0 & 0 & 0 & 0 \\
\hline
M46 & 0 & 0 & 0 & 0 & 0 & 0 & 0 & 0 & 0 & 1 & 0 & 0 & 0 & 0 & 0 & 0 \\
\hline
M47 & 1 & 1 & 0 & 0 & 0 & 0 & 0 & 0 & 1 & 0 & 0 & 0 & 0 & 0 & 0 & 0 \\
\hline
\end{tabular}
}
\end{table}

\subsection{Sizing Parameters}
\label{sec:sizing}

Table~\ref{tab:device_sizing} presents the device sizing parameters for the best discovered 8-mode power converter.

\begin{table}[H]
\centering
\caption{Device sizing for the best discovered 8-mode power converter topology. The circuit contains 10 capacitors (each C = 300pF, 390$\mu$m $\times$ 390$\mu$m) and 47 switches. Total area is 1.5mm$^2$, with capacitors occupying 99.88\% and switches occupying 0.12\%.}
\label{tab:device_sizing}
\adjustbox{max width=\textwidth, max height=0.15\textheight}{
\begin{tabular}{|c|c|c|c||c|c|c|c||c|c|c|c|}
\hline
Device & W & L & m & Device & W & L & m & Device & W & L & m \\
\hline
M1 & 10$\mu$ & 180n & 50 & M2 & 10$\mu$ & 180n & 33 & M3 & 10$\mu$ & 180n & 50 \\
\hline
M4 & 10$\mu$ & 180n & 32 & M5 & 10$\mu$ & 180n & 50 & M6 & 10$\mu$ & 180n & 19 \\
\hline
M7 & 10$\mu$ & 180n & 3 & M8 & 10$\mu$ & 180n & 20 & M9 & 10$\mu$ & 180n & 4 \\
\hline
M10 & 10$\mu$ & 180n & 12 & M11 & 10$\mu$ & 180n & 12 & M12 & 10$\mu$ & 180n & 26 \\
\hline
M13 & 10$\mu$ & 180n & 50 & M14 & 10$\mu$ & 180n & 1 & M15 & 10$\mu$ & 180n & 50 \\
\hline
M16 & 10$\mu$ & 180n & 12 & M17 & 10$\mu$ & 180n & 50 & M18 & 10$\mu$ & 180n & 2 \\
\hline
M19 & 10$\mu$ & 180n & 18 & M20 & 10$\mu$ & 180n & 1 & M21 & 10$\mu$ & 180n & 49 \\
\hline
M22 & 10$\mu$ & 180n & 11 & M23 & 10$\mu$ & 180n & 50 & M24 & 10$\mu$ & 180n & 2 \\
\hline
M25 & 10$\mu$ & 180n & 26 & M26 & 10$\mu$ & 180n & 9 & M27 & 10$\mu$ & 180n & 11 \\
\hline
M28 & 10$\mu$ & 180n & 10 & M29 & 10$\mu$ & 180n & 12 & M30 & 10$\mu$ & 180n & 35 \\
\hline
M31 & 10$\mu$ & 180n & 3 & M32 & 10$\mu$ & 180n & 1 & M33 & 10$\mu$ & 180n & 3 \\
\hline
M34 & 10$\mu$ & 180n & 2 & M35 & 10$\mu$ & 180n & 13 & M36 & 10$\mu$ & 180n & 13 \\
\hline
M37 & 10$\mu$ & 180n & 46 & M38 & 10$\mu$ & 180n & 50 & M39 & 10$\mu$ & 180n & 37 \\
\hline
M40 & 10$\mu$ & 180n & 2 & M41 & 10$\mu$ & 180n & 14 & M42 & 10$\mu$ & 180n & 11 \\
\hline
M43 & 10$\mu$ & 180n & 50 & M44 & 10$\mu$ & 180n & 43 & M45 & 10$\mu$ & 180n & 29 \\
\hline
M46 & 10$\mu$ & 180n & 26 & M47 & 10$\mu$ & 180n & 4 & & & &  \\
\hline
\end{tabular}}
\end{table}

\section{Best Training Topology Details}
\label{sec:appendix_training}

This section provides the complete details for the best-performing 8-mode reconfigurable power converter in the training set, which achieves FoM = 0.263 and serves as the baseline for evaluating PowerGenie's discovery capability.

\subsection{Topology}
\label{sec:topology_training}

Figure~\ref{fig:training_best_topology} shows the best training topology—a 51-switch, 11-capacitor, 8-mode reconfigurable power converter. Compared to the best discovered topology (47 switches, 10 capacitors), it uses 4 more switches and 1 more capacitor.

\begin{figure}[H]
    \centering
    \includegraphics[width=0.48\linewidth]{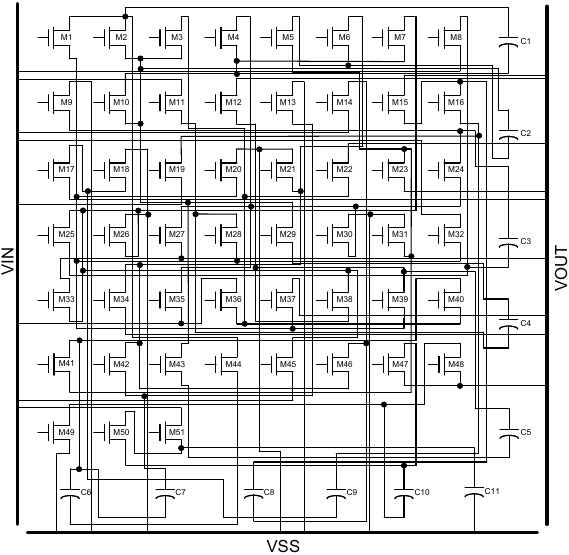}
    \caption{Best 8-mode power converter topology in the training set (FoM = 0.263). The topology contains 51 switches and 11 capacitors.}
    \label{fig:training_best_topology}
\end{figure}

\subsection{Control Scheme}
\label{sec:control_scheme_training}

Table~\ref{tab:training_best_control} presents the complete control scheme for the training-best 8-mode reconfigurable power converter. Each entry specifies whether a switch is on (1) or off (0) during each phase ($\Phi_1$ or $\Phi_2$) for each voltage conversion ratio (VCR) mode.

\begin{table}[H]
\centering
\caption{Training set best 8-mode power converter's control scheme}
\label{tab:training_best_control}
\adjustbox{max width=\textwidth, max height=0.44\textheight}{
\begin{tabular}{|c|c|c|c|c|c|c|c|c|c|c|c|c|c|c|c|c|}
\hline
 & \multicolumn{2}{c|}{VCR=1} & \multicolumn{2}{c|}{VCR=4/5} & \multicolumn{2}{c|}{VCR=3/4} & \multicolumn{2}{c|}{VCR=2/3} & \multicolumn{2}{c|}{VCR=1/2} & \multicolumn{2}{c|}{VCR=1/3} & \multicolumn{2}{c|}{VCR=1/4} & \multicolumn{2}{c|}{VCR=1/5} \\
\hline
 & $\Phi$2 & $\Phi$1 & $\Phi$2 & $\Phi$1 & $\Phi$2 & $\Phi$1 & $\Phi$2 & $\Phi$1 & $\Phi$2 & $\Phi$1 & $\Phi$2 & $\Phi$1 & $\Phi$2 & $\Phi$1 & $\Phi$2 & $\Phi$1 \\
\hline
M1 & 0 & 0 & 0 & 1 & 0 & 1 & 0 & 1 & 0 & 0 & 0 & 1 & 1 & 0 & 0 & 1 \\
\hline
M2 & 0 & 0 & 1 & 0 & 1 & 0 & 1 & 0 & 0 & 0 & 1 & 0 & 0 & 1 & 1 & 0 \\
\hline
M3 & 0 & 0 & 0 & 1 & 0 & 1 & 0 & 1 & 0 & 0 & 0 & 1 & 1 & 0 & 1 & 0 \\
\hline
M4 & 0 & 0 & 1 & 0 & 0 & 1 & 0 & 1 & 0 & 0 & 1 & 0 & 1 & 0 & 0 & 1 \\
\hline
M5 & 0 & 0 & 0 & 1 & 1 & 0 & 0 & 1 & 0 & 0 & 0 & 1 & 0 & 1 & 0 & 1 \\
\hline
M6 & 0 & 0 & 1 & 0 & 0 & 1 & 1 & 0 & 0 & 0 & 1 & 0 & 1 & 0 & 1 & 0 \\
\hline
M7 & 0 & 0 & 0 & 1 & 1 & 0 & 1 & 0 & 0 & 0 & 0 & 1 & 0 & 1 & 1 & 0 \\
\hline
M8 & 0 & 0 & 0 & 0 & 0 & 0 & 0 & 0 & 0 & 0 & 0 & 0 & 0 & 0 & 0 & 1 \\
\hline
M9 & 0 & 0 & 0 & 0 & 0 & 0 & 0 & 0 & 0 & 0 & 0 & 0 & 0 & 0 & 1 & 0 \\
\hline
M10 & 0 & 0 & 0 & 0 & 0 & 0 & 0 & 0 & 0 & 0 & 0 & 0 & 0 & 0 & 0 & 1 \\
\hline
M11 & 0 & 0 & 0 & 0 & 0 & 0 & 0 & 0 & 0 & 0 & 0 & 0 & 0 & 1 & 0 & 0 \\
\hline
M12 & 0 & 0 & 0 & 0 & 1 & 1 & 0 & 0 & 0 & 0 & 0 & 0 & 0 & 0 & 0 & 0 \\
\hline
M13 & 0 & 0 & 0 & 0 & 1 & 0 & 0 & 0 & 0 & 0 & 0 & 0 & 0 & 0 & 0 & 0 \\
\hline
M14 & 0 & 0 & 0 & 1 & 0 & 0 & 0 & 0 & 0 & 0 & 0 & 0 & 0 & 0 & 0 & 0 \\
\hline
M15 & 0 & 0 & 0 & 1 & 0 & 0 & 0 & 0 & 0 & 0 & 0 & 0 & 0 & 0 & 0 & 0 \\
\hline
M16 & 0 & 0 & 1 & 0 & 0 & 0 & 0 & 0 & 0 & 0 & 0 & 0 & 0 & 0 & 0 & 0 \\
\hline
M17 & 0 & 0 & 0 & 1 & 0 & 0 & 0 & 0 & 0 & 0 & 0 & 0 & 0 & 0 & 0 & 0 \\
\hline
M18 & 0 & 0 & 1 & 0 & 0 & 0 & 0 & 0 & 0 & 0 & 0 & 0 & 0 & 0 & 0 & 0 \\
\hline
M19 & 0 & 0 & 0 & 1 & 0 & 0 & 0 & 0 & 0 & 0 & 0 & 0 & 0 & 0 & 0 & 0 \\
\hline
M20 & 0 & 0 & 0 & 0 & 0 & 0 & 0 & 0 & 0 & 0 & 0 & 0 & 0 & 0 & 1 & 1 \\
\hline
M21 & 0 & 0 & 0 & 0 & 0 & 0 & 1 & 1 & 0 & 0 & 0 & 0 & 1 & 1 & 1 & 1 \\
\hline
M22 & 0 & 0 & 0 & 0 & 0 & 0 & 0 & 0 & 0 & 0 & 1 & 1 & 1 & 1 & 1 & 1 \\
\hline
M23 & 0 & 0 & 1 & 1 & 0 & 0 & 1 & 1 & 0 & 0 & 0 & 0 & 1 & 1 & 1 & 1 \\
\hline
M24 & 0 & 0 & 0 & 0 & 0 & 0 & 0 & 0 & 0 & 0 & 0 & 0 & 0 & 0 & 1 & 1 \\
\hline
M25 & 0 & 0 & 0 & 0 & 0 & 0 & 0 & 0 & 0 & 0 & 0 & 0 & 0 & 0 & 1 & 1 \\
\hline
M26 & 0 & 0 & 0 & 0 & 0 & 0 & 0 & 0 & 0 & 0 & 0 & 0 & 1 & 1 & 0 & 0 \\
\hline
M27 & 0 & 0 & 1 & 1 & 0 & 0 & 0 & 0 & 0 & 0 & 0 & 0 & 1 & 1 & 0 & 0 \\
\hline
M28 & 0 & 0 & 0 & 0 & 0 & 0 & 0 & 0 & 0 & 0 & 0 & 0 & 1 & 1 & 0 & 0 \\
\hline
M29 & 0 & 0 & 0 & 0 & 0 & 0 & 0 & 0 & 0 & 0 & 0 & 0 & 1 & 1 & 0 & 0 \\
\hline
M30 & 0 & 0 & 0 & 0 & 0 & 0 & 0 & 0 & 0 & 0 & 1 & 1 & 0 & 0 & 0 & 0 \\
\hline
M31 & 0 & 0 & 0 & 0 & 0 & 0 & 0 & 0 & 0 & 0 & 1 & 1 & 0 & 0 & 0 & 0 \\
\hline
M32 & 0 & 0 & 0 & 0 & 0 & 0 & 0 & 0 & 0 & 0 & 1 & 1 & 0 & 0 & 0 & 0 \\
\hline
M33 & 0 & 0 & 0 & 0 & 0 & 0 & 1 & 1 & 0 & 0 & 1 & 1 & 0 & 0 & 0 & 0 \\
\hline
M34 & 0 & 0 & 0 & 0 & 0 & 0 & 0 & 0 & 0 & 0 & 1 & 1 & 0 & 0 & 0 & 0 \\
\hline
M35 & 0 & 0 & 0 & 0 & 1 & 1 & 1 & 1 & 0 & 0 & 0 & 0 & 0 & 0 & 0 & 0 \\
\hline
M36 & 0 & 0 & 1 & 1 & 0 & 0 & 1 & 1 & 0 & 0 & 0 & 0 & 0 & 0 & 0 & 0 \\
\hline
M37 & 0 & 0 & 1 & 1 & 0 & 0 & 1 & 1 & 0 & 0 & 0 & 0 & 0 & 0 & 0 & 0 \\
\hline
M38 & 0 & 0 & 0 & 0 & 1 & 1 & 0 & 0 & 0 & 0 & 0 & 0 & 0 & 0 & 0 & 0 \\
\hline
M39 & 0 & 0 & 0 & 0 & 1 & 1 & 0 & 0 & 0 & 0 & 0 & 0 & 0 & 0 & 0 & 0 \\
\hline
M40 & 0 & 0 & 0 & 0 & 1 & 1 & 0 & 0 & 0 & 0 & 0 & 0 & 0 & 0 & 0 & 0 \\
\hline
M41 & 0 & 0 & 0 & 0 & 1 & 1 & 0 & 0 & 0 & 0 & 0 & 0 & 0 & 0 & 0 & 0 \\
\hline
M42 & 0 & 0 & 0 & 0 & 1 & 1 & 0 & 0 & 0 & 0 & 0 & 0 & 0 & 0 & 0 & 0 \\
\hline
M43 & 0 & 0 & 0 & 0 & 1 & 1 & 0 & 0 & 0 & 0 & 0 & 0 & 0 & 0 & 0 & 0 \\
\hline
M44 & 0 & 0 & 0 & 0 & 1 & 1 & 0 & 0 & 0 & 0 & 0 & 0 & 0 & 0 & 0 & 0 \\
\hline
M45 & 0 & 0 & 1 & 1 & 0 & 0 & 0 & 0 & 0 & 0 & 0 & 0 & 0 & 0 & 0 & 0 \\
\hline
M46 & 0 & 0 & 1 & 1 & 0 & 0 & 0 & 0 & 0 & 0 & 0 & 0 & 0 & 0 & 0 & 0 \\
\hline
M47 & 1 & 0 & 0 & 0 & 0 & 0 & 0 & 0 & 1 & 0 & 0 & 0 & 0 & 0 & 0 & 0 \\
\hline
M48 & 0 & 0 & 0 & 0 & 0 & 0 & 0 & 0 & 0 & 1 & 0 & 0 & 0 & 0 & 0 & 0 \\
\hline
M49 & 1 & 1 & 0 & 0 & 0 & 0 & 0 & 0 & 1 & 0 & 0 & 0 & 0 & 0 & 0 & 0 \\
\hline
M50 & 0 & 1 & 0 & 0 & 0 & 0 & 0 & 0 & 0 & 1 & 0 & 0 & 0 & 0 & 0 & 0 \\
\hline
M51 & 0 & 1 & 0 & 0 & 0 & 0 & 0 & 0 & 1 & 1 & 0 & 0 & 0 & 0 & 0 & 0 \\
\hline
\end{tabular}
}
\end{table}

\subsection{Sizing Parameters}
\label{sec:sizing_training}

Table~\ref{tab:training_device_sizing} presents the device sizing parameters for the training-best 8-mode power converter.

\begin{table}[H]
\centering
\caption{Device sizing for the training best 8-mode power converter topology. The circuit contains 11 capacitors (each C = 272.7pF, 370$\mu$m $\times$ 370$\mu$m) and 51 switches. Total area is 1.5mm$^2$, with capacitors occupying 99.84\% and switches occupying 0.16\%.}
\label{tab:training_device_sizing}
\adjustbox{max width=\textwidth}{
\begin{tabular}{|c|c|c|c||c|c|c|c||c|c|c|c|}
\hline
Device & W & L & m & Device & W & L & m & Device & W & L & m \\
\hline
M1 & 10$\mu$ & 180n & 38 & M2 & 10$\mu$ & 180n & 39 & M3 & 10$\mu$ & 180n & 41 \\
\hline
M4 & 10$\mu$ & 180n & 36 & M5 & 10$\mu$ & 180n & 38 & M6 & 10$\mu$ & 180n & 20 \\
\hline
M7 & 10$\mu$ & 180n & 50 & M8 & 10$\mu$ & 180n & 23 & M9 & 10$\mu$ & 180n & 4 \\
\hline
M10 & 10$\mu$ & 180n & 11 & M11 & 10$\mu$ & 180n & 1 & M12 & 10$\mu$ & 180n & 19 \\
\hline
M13 & 10$\mu$ & 180n & 4 & M14 & 10$\mu$ & 180n & 12 & M15 & 10$\mu$ & 180n & 50 \\
\hline
M16 & 10$\mu$ & 180n & 12 & M17 & 10$\mu$ & 180n & 14 & M18 & 10$\mu$ & 180n & 49 \\
\hline
M19 & 10$\mu$ & 180n & 50 & M20 & 10$\mu$ & 180n & 6 & M21 & 10$\mu$ & 180n & 28 \\
\hline
M22 & 10$\mu$ & 180n & 39 & M23 & 10$\mu$ & 180n & 34 & M24 & 10$\mu$ & 180n & 10 \\
\hline
M25 & 10$\mu$ & 180n & 7 & M26 & 10$\mu$ & 180n & 13 & M27 & 10$\mu$ & 180n & 50 \\
\hline
M28 & 10$\mu$ & 180n & 33 & M29 & 10$\mu$ & 180n & 47 & M30 & 10$\mu$ & 180n & 11 \\
\hline
M31 & 10$\mu$ & 180n & 16 & M32 & 10$\mu$ & 180n & 42 & M33 & 10$\mu$ & 180n & 35 \\
\hline
M34 & 10$\mu$ & 180n & 9 & M35 & 10$\mu$ & 180n & 50 & M36 & 10$\mu$ & 180n & 46 \\
\hline
M37 & 10$\mu$ & 180n & 34 & M38 & 10$\mu$ & 180n & 13 & M39 & 10$\mu$ & 180n & 15 \\
\hline
M40 & 10$\mu$ & 180n & 14 & M41 & 10$\mu$ & 180n & 7 & M42 & 10$\mu$ & 180n & 9 \\
\hline
M43 & 10$\mu$ & 180n & 14 & M44 & 10$\mu$ & 180n & 19 & M45 & 10$\mu$ & 180n & 24 \\
\hline
M46 & 10$\mu$ & 180n & 2 & M47 & 10$\mu$ & 180n & 44 & M48 & 10$\mu$ & 180n & 32 \\
\hline
M49 & 10$\mu$ & 180n & 28 & M50 & 10$\mu$ & 180n & 50 & M51 & 10$\mu$ & 180n & 47 \\
\hline
\end{tabular}}
\end{table}

\section{Additional Experimental Results}
\label{sec:appendix_extended}
This section provides additional experimental results to supplement the findings in Section~\ref{sec: Results}.

\subsection{Multi-Run Statistics}
\label{sec:multirun}
To provide statistical context, we conduct three runs for each preference alignment method. Table~\ref{tab:multirun_results} reports the mean and standard deviation across runs.

\begin{table}[H]
\centering
\caption{Multi-run results (mean $\pm$ std) across three runs.}
\begin{threeparttable}
\begin{tabular}{lcccc}
\toprule
Method & Syntax $\uparrow$ & Functional $\uparrow$ & Uniqueness $\uparrow$ & FoM $\uparrow$ \\
\midrule
Pretrain+DPO & 35.30$\pm$0.59 & 1.50$\pm$0.22 & 0.50$\pm$0.08 & 0.253$\pm$0.005 \\
Pretrain+SFT+DPO & 44.67$\pm$0.78 & 37.67$\pm$0.53 & 2.03$\pm$0.50 & 0.264$\pm$0.006 \\
Pretrain+PPO & 34.80$\pm$3.74 & 1.37$\pm$0.42 & 0.50$\pm$0.16 & 0.244$\pm$0.008 \\
Pretrain+SFT+PPO & 47.63$\pm$0.64 & 39.70$\pm$1.02 & 2.77$\pm$0.52 & 0.268$\pm$0.012 \\
\textbf{Pretrain+Evo} & \textbf{90.47$\pm$0.40} & \textbf{86.20$\pm$1.07} & \textbf{30.73$\pm$1.37} & \textbf{0.320$\pm$0.004} \\
\bottomrule
\end{tabular}
\begin{tablenotes}\scriptsize
\item All methods generate 1,000 topologies per run. Training set best FoM = 0.263.
\end{tablenotes}
\end{threeparttable}
\label{tab:multirun_results}
\end{table}

The results confirm that evolutionary finetuning consistently outperforms all baselines across runs. All three evolutionary finetuning runs discover topologies exceeding the training set maximum FoM of 0.263.

Table~\ref{tab:ablation_multirun} provides multi-run statistics for the ablation study presented in Section~\ref{sec: Results}.

\begin{table}[H]
\centering
\caption{Ablation study multi-run results (mean $\pm$ std) across three runs.}
\begin{threeparttable}
\begin{tabular}{lcccc}
\toprule
Method & Syntax $\uparrow$ & Functional $\uparrow$ & Uniqueness $\uparrow$ & FoM $\uparrow$ \\
\midrule
w/o Selection & 77.67$\pm$1.27 & 72.00$\pm$1.20 & 18.33$\pm$0.74 & 0.260$\pm$0.002 \\
w/o Generation & 86.80$\pm$0.70 & 82.23$\pm$0.29 & 0.67$\pm$0.12 & 0.263$\pm$0.001 \\
\textbf{PowerGenie} & \textbf{90.47$\pm$0.40} & \textbf{86.20$\pm$1.07} & \textbf{30.73$\pm$1.37} & \textbf{0.320$\pm$0.004} \\
\bottomrule
\end{tabular}
\begin{tablenotes}\scriptsize
\item All methods generate 1,000 topologies per run. Training set best FoM = 0.263.
\end{tablenotes}
\end{threeparttable}
\label{tab:ablation_multirun}
\end{table}

The multi-run statistics reinforce the ablation findings: removing selection reduces FoM to $0.260\pm0.002$, below the training set best of 0.263, while removing generation yields FoM matching but never exceeding this baseline ($0.263\pm0.001$). Both components are essential—selection drives exploitation toward high-quality designs, while generation enables exploration of novel topologies beyond the training distribution.

\subsection{FoM Distribution of Discovered Circuits}
\label{sec:fom_distribution}

Beyond pass rates, we analyze the FoM distribution of functionally valid circuits discovered by each method. Table~\ref{tab:fom_stats} reports the minimum, maximum, mean, and standard deviation of FoM values among all valid circuits generated in the best run.

\begin{table}[H]
\centering
\caption{FoM statistics of functionally valid circuits discovered by each method.}
\begin{threeparttable}
\begin{tabular}{lcccc}
\toprule
Method & Min $\uparrow$ & Max $\uparrow$ & Mean $\uparrow$ & Std $\downarrow$ \\
\midrule
Pretrain & $-$0.185 & 0.246 & 0.087 & 0.113 \\
Pretrain+SFT & $-$0.351 & 0.263 & 0.052 & 0.138 \\
Pretrain+DPO & $-$0.224 & 0.256 & 0.089 & 0.139 \\
Pretrain+SFT+DPO & $-$0.351 & 0.272 & 0.097 & 0.127 \\
Pretrain+PPO & $-$0.224 & 0.253 & 0.055 & 0.139 \\
Pretrain+SFT+PPO & $-$0.351 & 0.284 & 0.052 & 0.137 \\
\midrule
w/o Selection & $-$0.351 & 0.261 & 0.038 & 0.136 \\
w/o Generation & $-$0.026 & 0.263 & 0.193 & 0.052 \\
\midrule
\textbf{PowerGenie} & \textbf{0.104} & \textbf{0.323} & \textbf{0.222} & \textbf{0.033} \\
\bottomrule
\end{tabular}
\begin{tablenotes}\scriptsize
\item Statistics computed over all functionally valid 8-mode circuits from 1,000 generated topologies. Training set best FoM = 0.263.
\end{tablenotes}
\end{threeparttable}
\label{tab:fom_stats}
\end{table}

PowerGenie exhibits a fundamentally different FoM distribution compared to all baselines:

\paragraph{Higher Floor and Ceiling.} All baseline methods produce circuits with negative FoM values (as low as $-$0.351), indicating poor efficiency or excessive component counts. PowerGenie's minimum FoM (0.104) is strictly positive, ensuring no low-quality outputs. Simultaneously, PowerGenie achieves the highest maximum FoM (0.323), substantially surpassing the training set best (0.263) and the best baseline (SFT+PPO at 0.284).

\paragraph{Higher Mean.} PowerGenie's mean FoM (0.222) is 2--6$\times$ higher than all baselines (0.038--0.097), indicating that evolutionary finetuning shifts the entire output distribution toward high-quality designs rather than relying on tail samples.

\paragraph{Lower Variance.} PowerGenie's standard deviation (0.033) is 3--4$\times$ smaller than baselines ($\sim$0.13), demonstrating consistent generation of high-quality circuits. This tight distribution, combined with the elevated floor, ensures reliable performance in practice.



\end{document}